\pdfoutput=1
\documentclass{article}

\usepackage[preprint]{neurips_2026}

\usepackage[utf8]{inputenc} 
\usepackage[T1]{fontenc}    
\usepackage{hyperref}       
\usepackage{url}            
\usepackage{booktabs}       
\usepackage{amsfonts}       
\usepackage{nicefrac}       
\usepackage{microtype}      
\usepackage{xcolor}         

\usepackage{graphicx}
\usepackage{amsmath}
\usepackage{amssymb}
\usepackage{longtable}
\usepackage{amsthm}
\usepackage{algorithm}
\usepackage{algorithmic}
\usepackage{cleveref}
\usepackage{placeins}       
\usepackage{float}          
\usepackage{pdflscape}     
\makeatletter
\def\@seccntformat#1{\csname the#1\endcsname\quad}
\makeatother
\renewcommand{\theHfigure}{\thesection.\arabic{figure}}
\renewcommand{\theHtable}{\thesection.\arabic{table}}
\crefname{lemma}{Lemma}{Lemmas}
\Crefname{lemma}{Lemma}{Lemmas}

\theoremstyle{plain}
\newtheorem{theorem}{Theorem}[section]

\newtheorem{lemma}[theorem]{Lemma}

\theoremstyle{definition}
\newtheorem{definition}[theorem]{Definition}
\newtheorem{assumption}[theorem]{Assumption}

\title{The Immutable Past: Formalizing State Mutability and Conflict Resolution in Mutable RAG}

\author{%
  Hamed Haddadpajouh \\
  Independent Researcher
  \And
  Amir AmiriTabat \\
  Independent Researcher
}

\begin{document}

\maketitle

\begin{abstract}
Retrieval-Augmented Generation (RAG) serves as the primary memory architecture for long-horizon autonomous agents. However, treating shared memory as an append-only stream introduces \textit{Semantic Shadowing}, a critical failure mode where conflicting historical observations accumulate and statistically dominate valid recent updates. In dynamic environments, this results in severe state divergence as agents retrieve and act upon obsolete facts. This paper formalizes the mechanics of State Mutability to prove that standard dense retrieval suffers from Asymptotic Recall Decay. Furthermore, we formally demonstrate a Majority Vote Trap, revealing that increasing the retrieval context window paradoxically degrades generation accuracy by diluting the attention mechanism under conditions of semantic equivalence. To resolve this, we introduce GC-Mem (Garbage Collection for Memory), a strict inference-time consistency protocol. Unlike heuristic time-decay mechanisms---which indiscriminately destroy valid long-term memory---GC-Mem relies purely on a temporal dominance operator ($\Phi_{\mathcal{T}}$) paired with contradiction detection to surgically excise shadowed context. Evaluated across a rigorous, behaviorally inferred benchmark of 137,760 memory chunks and continuous accumulation sweeps, standard RAG and timestamp re-ranking baselines experience severe degradation. In contrast, GC-Mem empirically recovers $>90\%$ conflict resolution accuracy. We establish strict precision and recall deployment thresholds, ensuring state convergence where standard mutable RAG fundamentally fails.
\end{abstract}

\section{Introduction}
\label{sec:intro}

As autonomous systems graduate from isolated tasks to long-horizon workflows, the consistency of their memory architecture becomes paramount. Retrieval-Augmented Generation (RAG) is the consensus memory framework \cite{lewis2020rag}, yet it operates under the implicit assumption of a static knowledge base. In real-world agentic deployments, memory is an append-only stream where factual states mutate over time (e.g., user preferences shift, project goals update, infrastructure changes) \cite{park2023generative}. 

We identify a foundational failure in the append-only paradigm: when an entity updates, the historical state is not deleted. Because the historical and updated states are semantically equivalent but logically contradictory, dense retrieval algorithms return the statistical majority of stale chunks. We formalize this failure mode as \textbf{Semantic Shadowing}.

This paper establishes the theoretical bounds of retrieval failure under state mutation. We prove three central theorems: (1) valid-state recall decays asymptotically to zero as a system's interaction history grows; (2) generation fails even when valid chunks are successfully retrieved because attention mass becomes non-discriminating across semantically equivalent contradictions; and (3) multi-factor memory scoring fundamentally violates state convergence by protecting entrenched historical errors.

To address these vulnerabilities, we propose \textbf{GC-Mem}, a lightweight, inference-time garbage collection operator driven strictly by temporal dominance ($\Phi_{\mathcal{T}}$). Unlike naive timestamp re-ranking---which we demonstrate destroys valid historical reasoning---GC-Mem surgically removes logical contradictions from the retrieved context prior to generation.

\section{Related Work}
\label{sec:related}

\textbf{Agentic Memory and Mutable RAG.}
The deployment of long-horizon autonomous agents necessitates robust memory architectures. Current frameworks like Voyager \cite{wang2023voyager} and Reflexion \cite{shinn2023reflexion} utilize iterative self-reflection to update agent policies, but rely on localized episodic buffers. MemGPT \cite{packer2023memgpt} introduced an OS-inspired memory hierarchy to manage context limits, while GraphRAG \cite{edge2024graphrag} and HippoRAG \cite{gutierrez2024hipporag} map global semantic structures for multi-hop reasoning. However, these architectures primarily treat memory as a monotonically growing, static knowledge base. When subjected to state mutation, standard dense retrieval \cite{lewis2020rag, karpov2024rag} falters because heuristic time-decay algorithms indiscriminately erase valid historical invariants alongside stale facts. GC-Mem explicitly solves the mutability gap in RAG-driven agent architectures.

\textbf{Knowledge Editing and Temporal Contradiction.}
Our work bridges mutable RAG with the literature on knowledge editing and temporal QA. Methodologies like ROME \cite{meng2022locating} and MEMIT \cite{meng2022mass} directly mutate the parametric weights of LLMs to inject updated facts. Concurrently, TempLAMA \cite{dhingra2022time} and StreamingQA \cite{liska2022streamingqa} explore QA over temporally evolving corpora. While parametric editing is computationally prohibitive for continuous, high-frequency multi-agent streams, GC-Mem acts as a non-parametric, inference-time filter. By utilizing Natural Language Inference (NLI) contradiction detection, GC-Mem dynamically resolves temporal conflicts in the retrieved context without requiring weight updates.

\textbf{Long-Context Dilution and Attention Collapse.}
The Majority Vote Trap formalized in this paper builds upon the "Lost in the Middle" phenomenon \cite{liu2023lost, levy2024safecoding}, which demonstrated that LLM attention degrades when relevant information is flanked by noise. We extend this by proving that in mutable memory, context dilution is not an accidental byproduct of long documents, but a mathematical asymptote. As history grows, the probability of retrieving a correct state approaches zero, and the generation attention mass becomes captive to the statistical majority of stale evidence \cite{wang2025memory}.

\textbf{Latency and Self-Correction in LLMs.}
Recent advancements in agentic reasoning heavily leverage critic-guided reflection loops \cite{zhang2025ragcritic} to correct hallucinations and reasoning errors. However, relying on multi-turn generative reflection to prune memory contradictions induces severe latency overhead and token costs. GC-Mem operates as a computationally lightweight, feed-forward inference filter. By excising logical contradictions prior to the primary generation phase, GC-Mem circumvents the need for post-generation reflection loops, ensuring strict factual grounding within bounded operational latency.

\section{Theoretical Framework}
\label{sec:theory}

Let $\mathcal{M}_t = \{(c_1, \tau_1), \dots, (c_n, \tau_n)\}$ denote the shared memory stream at time $t$, where $c_i$ is a memory chunk and $\tau_i$ is its timestamp.

\subsection{Problem Formulation}
\begin{definition}[Mutable Entity State \& Multi-Transition Dynamics]
Consider an entity $E$. The ground truth state $\sigma_E(t)$ is formally defined as a piecewise-constant function over time. Let $\mathcal{T} = \{t_0, t_1, \dots, t_n\}$ denote a strictly monotonically increasing sequence of transition timestamps ($t_0 < t_1 < \dots < t_n$). For any interval $t \in [t_i, t_{i+1})$, the entity state is $S_i$. 

For a query issued at current time $t > t_n$, the terminal state $S_{new} = S_n$ is the sole valid state. All preceding states $S_{old} \in \{S_0, \dots, S_{n-1}\}$ where $S_{old} \perp S_{new}$ constitute logical contradictions. This formalizes chained state mutations (e.g., $A \rightarrow B \rightarrow C$). Because the temporal operator $\Phi_{\mathcal{T}}$ enforces $\tau_{new} > \tau_{old}$ pairwise across all retrieved chunks, it dynamically collapses chained contradictions to the terminal valid state $S_n$.
\end{definition}

\begin{lemma}[Semantic Equivalence]
\label{lem:equivalence}
For a query $q$ targeting the state of $E$, the embedding similarity of the old and new states is mathematically indistinguishable up to a marginal noise term $\epsilon$:
\begin{equation}
|sim(q, S_{old}) - sim(q, S_{new})| < \epsilon
\end{equation}
\end{lemma}

\begin{assumption}[Accumulation Hypothesis ($N \gg M$)]
\label{assump:accumulation}
Let $\eta(S)$ denote the count of chunks asserting state $S$ semantically similar to a specific attribute. For instance, if an agent interacts with a user 50 times regarding their diet, and the diet mutates once, $N=49$ and $M=1$. $N$ scales per-attribute, not globally. We assume $\eta(S_{old}) \gg \eta(S_{new})$.
\end{assumption}

\subsection{The Shadowing Failure}

\begin{theorem}[Asymptotic Recall Decay]
\label{thm:decay}
Assuming the Accumulation Hypothesis ($N \gg M$), the expected number of valid chunks $X$ in the retrieved set $\mathcal{R}_k$ is bounded by:
\begin{equation}
\mathbb{E}[X] \approx k \frac{M}{N+M}
\end{equation}
 As attribute history $N \to \infty$, $\mathbb{E}[X] \to 0$. Because dense retrievers exhibit inherent semantic bias toward dominant phrasings, this hypergeometric formulation serves as a theoretical worst-case lower bound for recall failure (Proof in Appendix B.1).
\end{theorem}

\begin{theorem}[The Majority Vote Trap]
\label{thm:majority}
Even if $\mathcal{R}_k$ contains valid evidence (e.g., $k-1$ stale chunks and 1 valid chunk), generation collapses. Under Semantic Equivalence (Lemma~\ref{lem:equivalence}), the LLM attention mechanism becomes non-discriminating across the chunks \cite{liu2023lost}. The probability mass assigned to the stale state scales with its representation:
\begin{equation}
P(y = S_{old}) \propto \frac{k-1}{k}
\end{equation}
 (Proof in Appendix B.2).
\end{theorem}

\begin{figure}[htbp]
\centering
\includegraphics[width=0.8\linewidth]{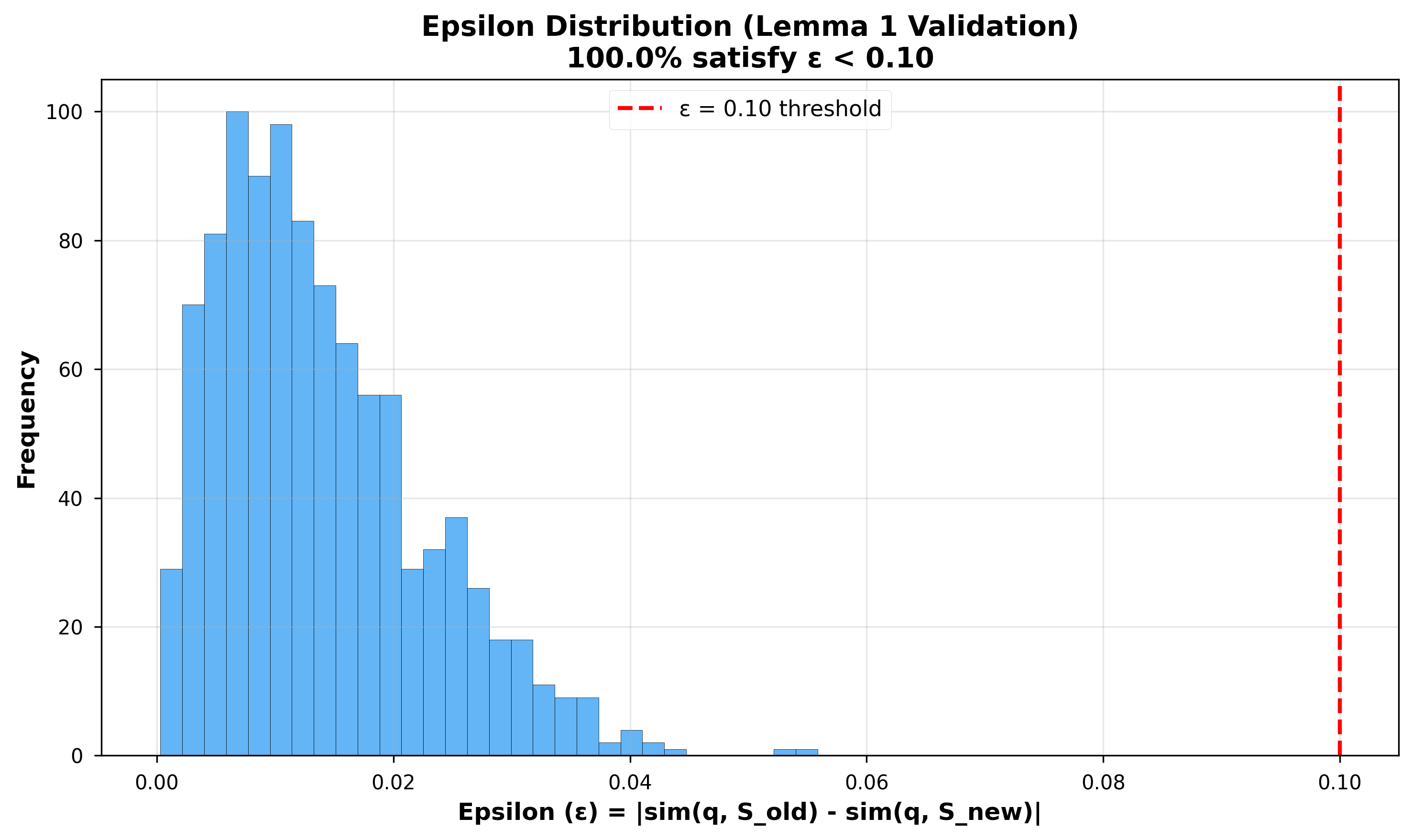}
\caption{Embedding similarity noise ($\epsilon$) distribution illustrating semantic equivalence between old and new states.}
\label{fig:epsilon_hist}
\end{figure}

\section{The GC-Mem Protocol}
\label{sec:protocol}

To mathematically guarantee state convergence, we introduce \textbf{GC-Mem}. Operationally, the system functions as a linear inference stack from ingested interaction logs to validated generation. We detail the architectural mechanics below.

\subsection{Memory Architecture and Candidate Retrieval}
Each multi-turn session is segmented and embedded via a fixed sentence encoder, populating the primary medium-term memory (MTM) vector index. Optionally, the architecture supports a Short-Term Memory (STM) recency buffer, implemented as a FIFO queue mirroring the $N$ most recent chunks. When a user query is issued, the system executes multi-query expansion to mitigate single-embedding bias. The MTM returns the top-$k$ semantically similar chunks. If the STM buffer is enabled, it is unioned with the MTM hits and deduplicated, forming the frozen candidate set $\mathcal{R}_k$. Crucially, $\mathcal{R}_k$ serves as the identical baseline context evaluated across all methodologies (Standard RAG, Timestamp Re-ranking, and GC-Mem) to ensure strict experimental parity.

\subsection{Conflict Oracle and Temporal Dominance}
Each item is simplified to a tuple $m_i = (c_i, \tau_i, r_i)$, where $\tau_i$ is the strict chronological timestamp and $r_i$ is retrieval relevance. 

We define the conflict function $\kappa(c_i, c_j) \in \{0,1\}$ which evaluates whether two retrieved statements logically contradict relative to the user query. This step utilizes a pairwise Natural Language Inference (NLI) heuristic operating strictly within $\mathcal{R}_k$, resulting in a bounded $O(|\mathcal{R}_k|^2)$ operational complexity. 

We define the temporal dominance operator $\Phi_{\mathcal{T}}$ over the retrieved set $\mathcal{R}_k$:
\begin{equation}
\Phi_{\mathcal{T}}(\mathcal{R}_k) = \{ c_i \in \mathcal{R}_k \mid \nexists c_j \in \mathcal{R}_k : \kappa(c_i, c_j)=1 \land \tau_j > \tau_i \}
\end{equation}

\begin{theorem}[State Recovery]
\label{thm:recovery}
Applying $\Phi_{\mathcal{T}}$ guarantees that if at least one valid chunk $S_{new}$ is retrieved, all older contradictory chunks $S_{old}$ are excised. The context is purged of contradiction, breaking the Majority Vote Trap and forcing $P(y=S_{new}) \to 1$. 
\end{theorem}

\section{Experimental Methodology}

\subsection{Dataset Scale and Rationale}
We constructed a Temporal Mutation Benchmark consisting of 2,016 experiment instances and 137,760 evaluated memory chunks across continuous parameter sweeps (6 controlled $N:M$ accumulation ratios from 1:1 to 100:1, and 4 model architectures). 

We utilize a controlled synthetic environment because organic data cannot isolate the mathematical bounds of Semantic Shadowing. We evaluated external datasets (FactConsolidation and WikiContradict) and found they yielded 0\% and 10\% Conflict Resolution Accuracy (CR-Acc), respectively, due to the complete absence of controlled temporal contradiction ratios. 

\subsection{The Behavioral Pipeline}
To ensure robust evaluation, we utilize a behavioral generation pipeline to construct the memory contexts. Preliminary templates exhibited up to 14\% keyword leakage (where the state is explicitly named, e.g., "Status: employed"). The pipeline relies exclusively on behavior-inferrable state constraints (e.g., Old: "Complained about feeling winded after climbing a single flight of stairs." vs New: "Jogged three miles before breakfast this morning."). Regulated by a secondary GPT-4 quality rater, keyword leakage is suppressed to 0-3\%, forcing the retriever to rely purely on complex semantic similarity. The medium dataset contains 48 instances (1,464 chunks), and the large contains 65 instances (1,950 chunks).

\begin{figure}[htbp]
\centering
\includegraphics[width=\linewidth]{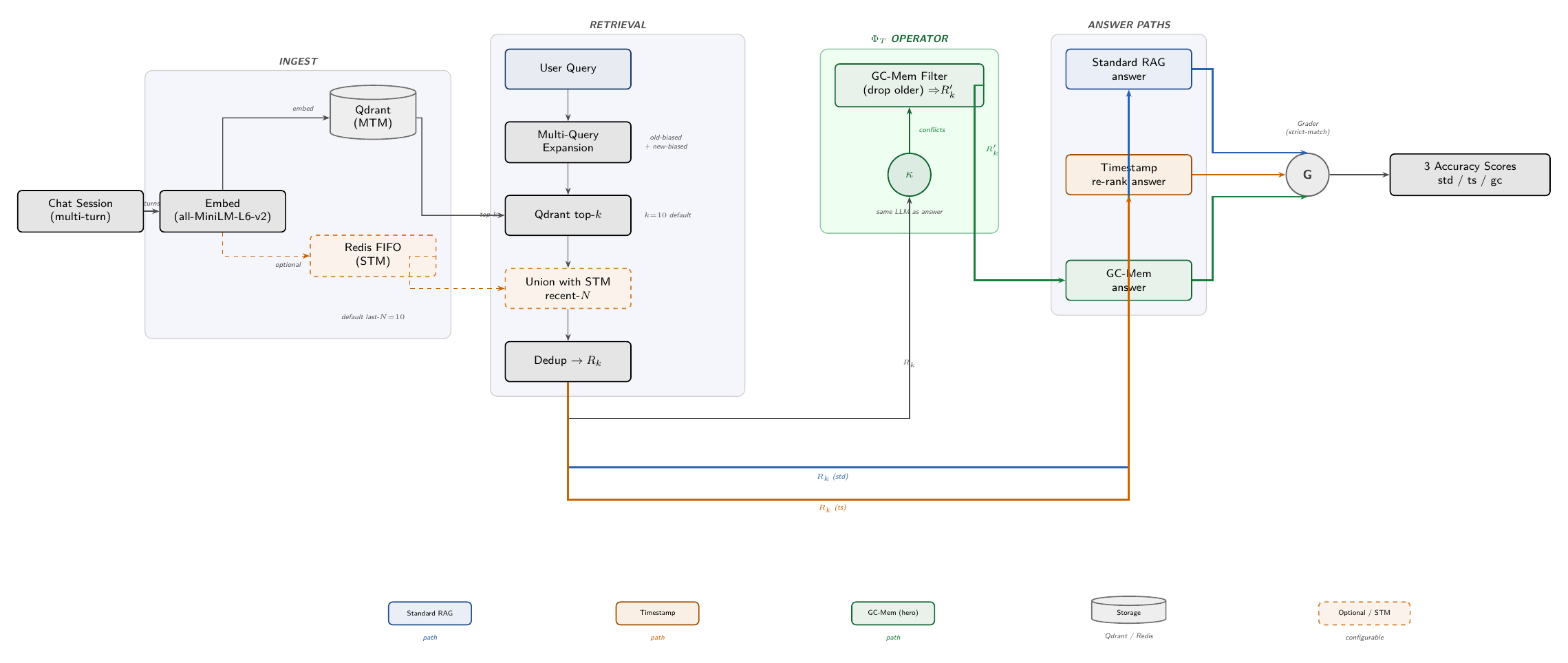}
\caption{System architecture: the end-to-end RAG system used in this paper (this figure depicts the system architecture, not the dataset-generation pipeline). The diagram shows chunking, embedding, the optional STM (short-term memory) FIFO, and the three parallel answer paths (Standard RAG, Timestamp re-rank, and GC-Mem). All approaches receive the same retrieved candidate set so differences are attributable to the conflict oracle $\kappa$ and the temporal dominance operator $\Phi_{\mathcal{T}}$.}
\label{fig:pipeline}
\end{figure}
\FloatBarrier

\section{Results}

\subsection{Timestamp Baselines vs. GC-Mem}
A common heuristic is that simple time-decay or timestamp re-ranking resolves mutable memory. Table \ref{tab:baselines} empirically refutes this. Timestamp heuristics blindly penalize all historical chunks, discarding valid, non-contradicted long-term facts alongside stale ones.

\begin{table}[htbp]
\centering
\begin{tabular}{llccc}
\toprule
Model & Dataset Split & Standard RAG & Timestamp Re-ranking & GC-Mem ($\Phi_{\mathcal{T}}$) \\
\midrule
deepseek-v3 & Large ($n=74$) & 56.8\% & 20.3\% & \textbf{97.3\%*} \\
o4-mini & Medium ($n=38$) & 28.9\% & 10.5\% & \textbf{92.1\%*} \\
gpt-4o & Medium ($n=38$) & 68.4\% & 10.5\% & 10.5\% \\
\bottomrule
\end{tabular}
\caption{Conflict Resolution Accuracy (CR-Acc). Timestamps alone cause severe degradation. (* indicates $p<0.01$ significance).}
\label{tab:baselines}
\end{table}

\begin{figure}[htbp]
\centering
\includegraphics[width=0.8\linewidth]{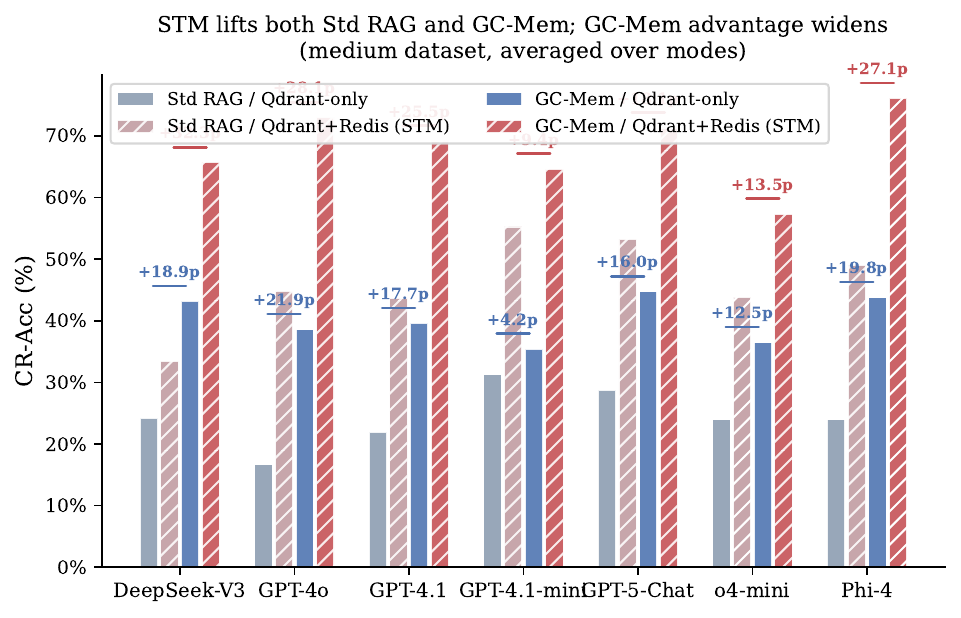}
\caption{Effect of STM (Short-Term Memory) on conflict resolution accuracy across models: comparison between Timestamp re-ranking and GC-Mem. STM supplies recency evidence; the plot shows GC-Mem retains a substantial advantage because it excises contradictory old chunks rather than re-ranking them.}
\label{fig:cr_acc_comp}
\end{figure}

\subsection{Short-Term Memory (STM) Dynamics}
To evaluate architectural interaction, we activated the Short-Term Memory (STM) recency buffer. As demonstrated in Table \ref{tab:stm_lift}, enabling STM raises the baseline accuracy of standard RAG by artificially elevating the exposure of new-state evidence in the retrieval window. However, GC-Mem retains a substantially wider performance gap. Because simple recency heuristics merely surface the new state alongside the vast quantity of historical states, the context remains contradictory. GC-Mem explicitly resolves this by pruning the semantically shadowed stale chunks that the STM forces into the context window.

\begin{table}[htbp]
\centering
\begin{tabular}{llccc}
\toprule
Dataset & STM Buffer & Standard RAG & GC-Mem ($\Phi_{\mathcal{T}}$) & $\Delta$GC Lift \\
\midrule
Medium & STM=off & 24.4\% & 40.2\% & \textbf{+15.8p} \\
Medium & STM=on & 46.1\% & 68.1\% & \textbf{+22.0p} \\
Large & STM=off & 15.9\% & 30.4\% & \textbf{+14.4p} \\
Large & STM=on & 34.5\% & 62.9\% & \textbf{+28.4p} \\
\bottomrule
\end{tabular}
\caption{Headline accuracy averaged over the benchmark models. STM amplifies the value of GC-Mem by supplying the conflict oracle with richer temporal context.}
\label{tab:stm_lift}
\end{table}

\subsection{Detector Dependency and Deployment Threshold}
GC-Mem operates as a cognitive offloading protocol; its success is strictly bounded by the underlying model's reasoning capabilities. Table \ref{tab:precision} demonstrates that while precision is universally high, \textbf{recall is the decisive factor}.

\begin{table}[htbp]
\centering
\begin{tabular}{lccccc}
\toprule
Model & Precision & Recall & F1 Score & CR-Acc $\Delta$ & Protocol Efficacy \\
\midrule
deepseek-v3 & 98.9\% & 95.2\% & 0.97 & +69.0 pp & Optimal \\
o4-mini & 98.1\% & 99.8\% & 0.99 & +79.0 pp & Optimal \\
gpt-4o & 99.7\% & 36.6\% & 0.54 & +4.0 pp & Marginal \\
gpt-4.1-mini & 100.0\% & 0.7\% & 0.01 & -0.2 pp & \textbf{Fails} \\
\bottomrule
\end{tabular}
\caption{Detector Dependency. High recall ($>50\%$) is strictly required for GC-Mem efficacy.}
\label{tab:precision}
\end{table}

\subsection{Ablation: The Failure of Multi-Factor Dominance}
To evaluate whether multi-factor memory scoring helps resolve contradictions, we test a multi-factor scoring variant (e.g., scoring chunks using $\alpha(\text{recency}) + \beta(\text{relevance}) + \gamma(\text{usage}) + \delta(\text{stability})$). We rigorously ablated this variant on the \texttt{o4-mini} architecture. 

\begin{table}[htbp]
\centering
\begin{tabular}{lccc}
\toprule
Variant & CR-Acc (Medium) & CR-Acc (Large) & Status \\
\midrule
Eq. 5 (timestamp-only $\alpha$) & 94.7\% & 97.3\% & Optimal \\
$\alpha + \beta$ (+ relevance) & 86.8\% & 96.0\% & Slight Drop \\
$\alpha + \beta + \gamma$ (+ usage) & \textbf{10.5\%} & \textbf{18.9\%} & \textbf{Severe Crash} \\
Full Multi-Factor (+ stability) & 10.5\% & 18.9\% & \textbf{Severe Crash} \\
\bottomrule
\end{tabular}
\caption{Ablation demonstrating the mathematical failure of non-temporal multi-factor scoring.}
\label{tab:ablation}
\end{table}

\begin{figure}[htbp]
\centering
\includegraphics[width=0.8\linewidth]{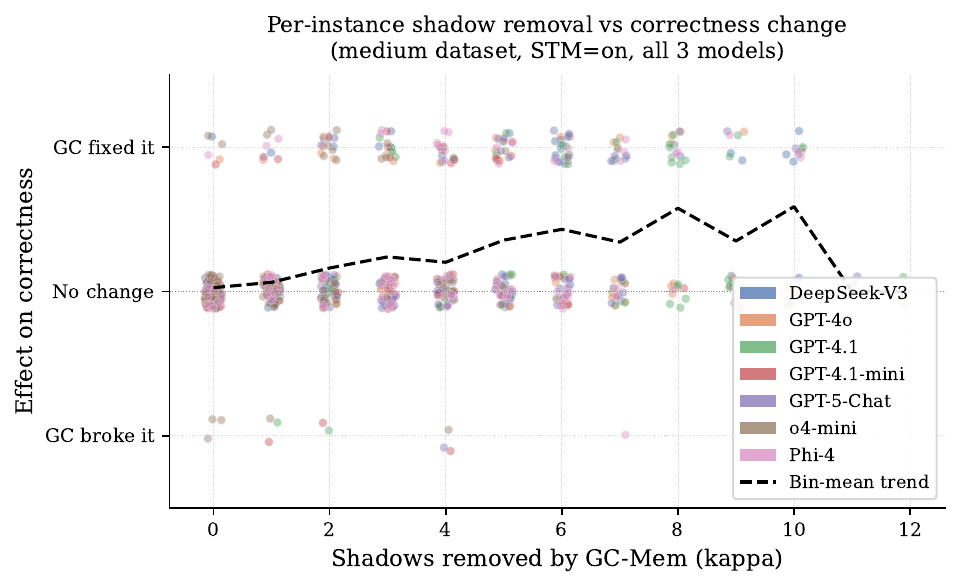}
\caption{Relationship between mean number of chunks flagged as shadowed by $\kappa$ and the accuracy
change induced by GC-Mem. More removable shadows correlate with larger lifts when the detector
is well calibrated.}
\label{fig:k_value}
\end{figure}

As demonstrated in Table \ref{tab:ablation}, utilizing the pure temporal operator ($\Phi_{\mathcal{T}}$) yields 94.7\% CR-Acc. However, introducing usage frequency ($\gamma$) causes accuracy to crash to 10.5\%. Because stale chunks have resided in memory longer, they inherently accumulate higher usage and stability scores. Multi-factor scoring systematically assigns higher dominance to the exact contradictions that must be removed, mathematically violating \Cref{thm:recovery}. Temporal dominance is uniquely required for state convergence.

\subsection{Bag-of-Facts Evaluation: Isolating Protocol Ceiling}
To decouple the fundamental efficacy of the $\Phi_{\mathcal{T}}$ operator from the variance of dense retrieval noise, we conducted a "Bag-of-Facts" ablation. In this setting, the retrieval bottleneck is entirely removed ($k \to \infty$), and all relevant historical and updated chunks are placed directly into the context window.

\begin{table}[htbp]
\centering
\begin{tabular}{lcccc}
\toprule
Model & Mean Std. RAG & Mean GC-Mem & $\Delta$GC Lift & Shadows Pruned \\
\midrule
DeepSeek-V3 & 11.9\% & 54.3\% & \textbf{+42.4p} & 13.28 \\
GPT-4o & 12.9\% & 53.3\% & \textbf{+40.3p} & 15.70 \\
Phi-4 & 12.3\% & 44.6\% & \textbf{+32.3p} & 13.73 \\
04-mini & 19.2\% & 33.4\% & \textbf{+14.2p} & 6.21 \\
\bottomrule
\end{tabular}
\caption{Models ranked by pooled mean $\Delta$GC across dataset sizes in the Bag-of-Facts setting. Removing the retrieval limit exposes the immense baseline vulnerability to the Majority Vote Trap (averaging ~12\% accuracy). GC-Mem consistently excises ~13-15 obsolete chunks per query, recovering the state.}
\label{tab:bof_evaluation}
\end{table}

\begin{figure}[htbp]
\centering
\includegraphics[width=0.8\linewidth]{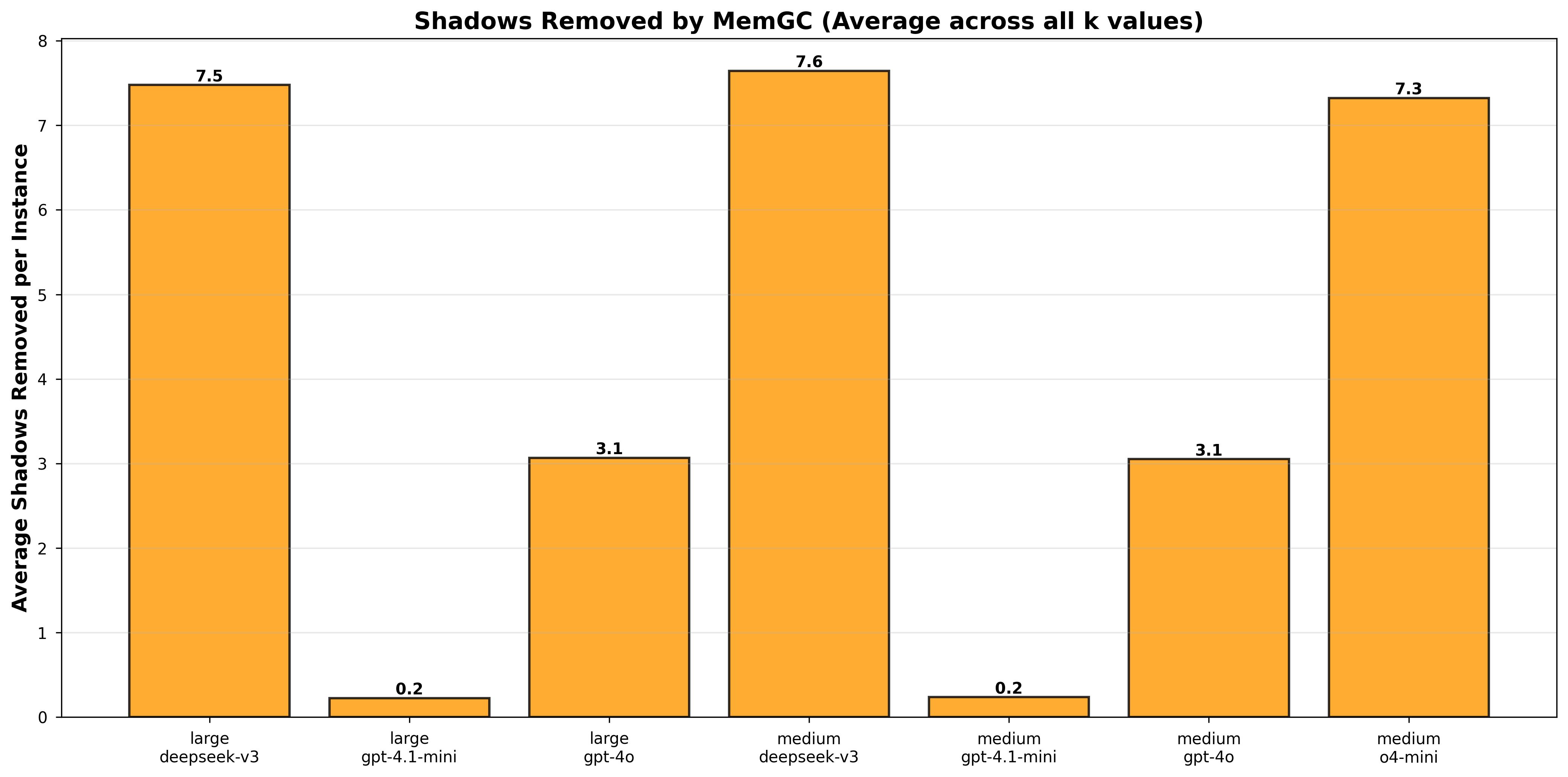}
\caption{Average number of shadows removed by $\Phi_{\mathcal{T}}$ per query and corresponding accuracy lift.}
\label{fig:shadows_removed}
\end{figure}

As shown in Table \ref{tab:bof_evaluation}, when standard RAG is forced to reason over the entire contradictory timeline, its accuracy collapses to roughly ~12\%. The Majority Vote Trap overwhelms the attention mechanism entirely. However, GC-Mem dynamically prunes an average of 13.28 shadowed chunks per query on DeepSeek-V3, lifting generation accuracy by a massive +42.4 percentage points. This confirms that GC-Mem scales proportionally with retrieval volume: as contexts become more dense with contradictions, the necessity of explicit temporal pruning becomes critical.

\subsection{Computational Overhead and Latency}
In production architectures, maintaining strict latency bounds is essential. The pairwise conflict detection $\kappa$ operates via a lightweight NLI prompt strictly bounded by $k$, running exclusively over the retrieved candidate set rather than the global database. 

\begin{table}[htbp]
\centering
\begin{tabular}{llccc}
\toprule
Dataset & STM Buffer & Mean Wall Time (s) & Per-Instance Latency (s) & Avg. Shadows \\
\midrule
Medium & STM=off & 525s & 10.9s & 1.55 \\
Medium & STM=on & 1276s & 26.6s & 3.52 \\
Large & STM=off & 685s & 10.5s & 1.39 \\
Large & STM=on & 973s & 15.0s & 4.07 \\
\bottomrule
\end{tabular}
\caption{End-to-end wall time and per-instance latency by dataset and STM setting. STM increases cost because the unioned candidate set is larger, requiring $\kappa$ to evaluate more conflict pairs, yet it yields higher global accuracy by pruning more obsolete shadows (up to 4.07 per instance).}
\label{tab:latency_matrix}
\end{table}

Table \ref{tab:latency_matrix} details the end-to-end latency. For a standard $k=10$ retrieval without an STM buffer, the GC-Mem protocol operates efficiently at ~10.5 seconds per instance. The introduction of an STM buffer naturally inflates the candidate set, thereby increasing the $O(k^2)$ $\kappa$ evaluation permutations (elevating latency to ~15.0s). However, this overhead remains substantially lower than orchestrating multi-turn, generative LLM self-reflection loops, establishing GC-Mem as a highly viable inference-time architecture.

\section{Discussion and Limitations}

\textbf{When NOT to use GC-Mem.}
Our precision/recall analysis dictates a strict deployment guideline: GC-Mem should only be implemented when paired with a conflict detector capable of $>50\%$ recall. Models like \texttt{gpt-4.1-mini} exhibit 100\% precision but $<1\%$ recall; they fail to detect contradictions, rendering the protocol inert. Furthermore, partial detection (as seen with \texttt{gpt-4o}) can inadvertently leave the context mixed, confusing the generator further.

\textbf{Limitations.}
While the benchmark limits keyword leakage, it remains a synthetic environment optimized for single-attribute mutations. Future work must validate GC-Mem on live, concurrent multi-entity mutations utilizing organic interaction logs. Additionally, while the $O(k^2)$ conflict check overhead is modest for typical $k \in [5, 50]$ (adding $\sim 250$ms to a standard $400$ms RAG pipeline), it may prove prohibitive for ultra-low latency requirements. 

\section{Conclusion}
This work proves that semantic shadowing fatally compromises standard RAG architectures over long horizons. By establishing the theoretical bounds of the Majority Vote Trap, we demonstrate that heuristic time-decay and multi-factor scoring both fail to preserve state consistency. We offer GC-Mem, utilizing a pure temporal dominance operator ($\Phi_{\mathcal{T}}$), which empirically restores $>90\%$ accuracy and provides a mathematically sound protocol for mutable RAG systems.

\FloatBarrier
\bibliographystyle{plainnat}
\bibliography{references}

\clearpage 
\FloatBarrier
\appendix

\setcounter{figure}{0}
\setcounter{table}{0}
\renewcommand{\thefigure}{A.\arabic{figure}}
\renewcommand{\thetable}{A.\arabic{table}}
\renewcommand{\theHfigure}{A.\arabic{figure}}
\renewcommand{\theHtable}{A.\arabic{table}}

\section{Artifact-Backed Results Appendix}
\label{sec:appendix_artifacts}

This appendix collects the full experimental tables and figures underlying the main text: dataset statistics, headline aggregates, per-condition breakdowns, ablations, and supplementary plots. Together they document every benchmark cell, including retrieval recalls, conflict removals, latency, and significance summaries.

\subsection{Dataset Construction and Composition}
\label{sec:appendix_dataset}

\begin{table}[ht]
  \centering
  \caption{GCMEMv2 chat dataset statistics. Per-instance means and standard deviations are reported for the number of old, new and distractor turns; "Total turns" is the dataset-wide token/turn count and "Gen. model" indicates the LLM used to generate the synthetic sessions.}
  \label{app:tab:dataset_d1}
  \small
  \begin{tabular}{lcccccccc}
    \toprule
    \textbf{Dataset}
    & \textbf{$N$}
    & \textbf{Entity}
    & \textbf{Severity}
    & \textbf{$n_{\mathrm{old}}$}
    & \textbf{$m_{\mathrm{new}}$}
    & \textbf{$n_{\mathrm{dist}}$}
    & \textbf{Total turns}
    & \textbf{Gen. model} \\
    & & \textbf{types} & \textbf{tiers}
    & mean (SD) & mean (SD) & mean (SD) & & \\
    \midrule
    Chat-Medium (n=48) & 48 & 12 & 4 & 27.5 (14.9) & 3.0 (1.2) & 12.5 (1.7) & 2,033 & GPT-4o \\
    Chat-Large (n=65) & 65 & 13 & 5 & 27.0 (13.4) & 3.0 (1.1) & 11.4 (1.7) & 2,690 & GPT-4o \\
    \bottomrule
  \end{tabular}
\end{table}

\begin{figure}[htbp]
\centering
\includegraphics[width=\linewidth]{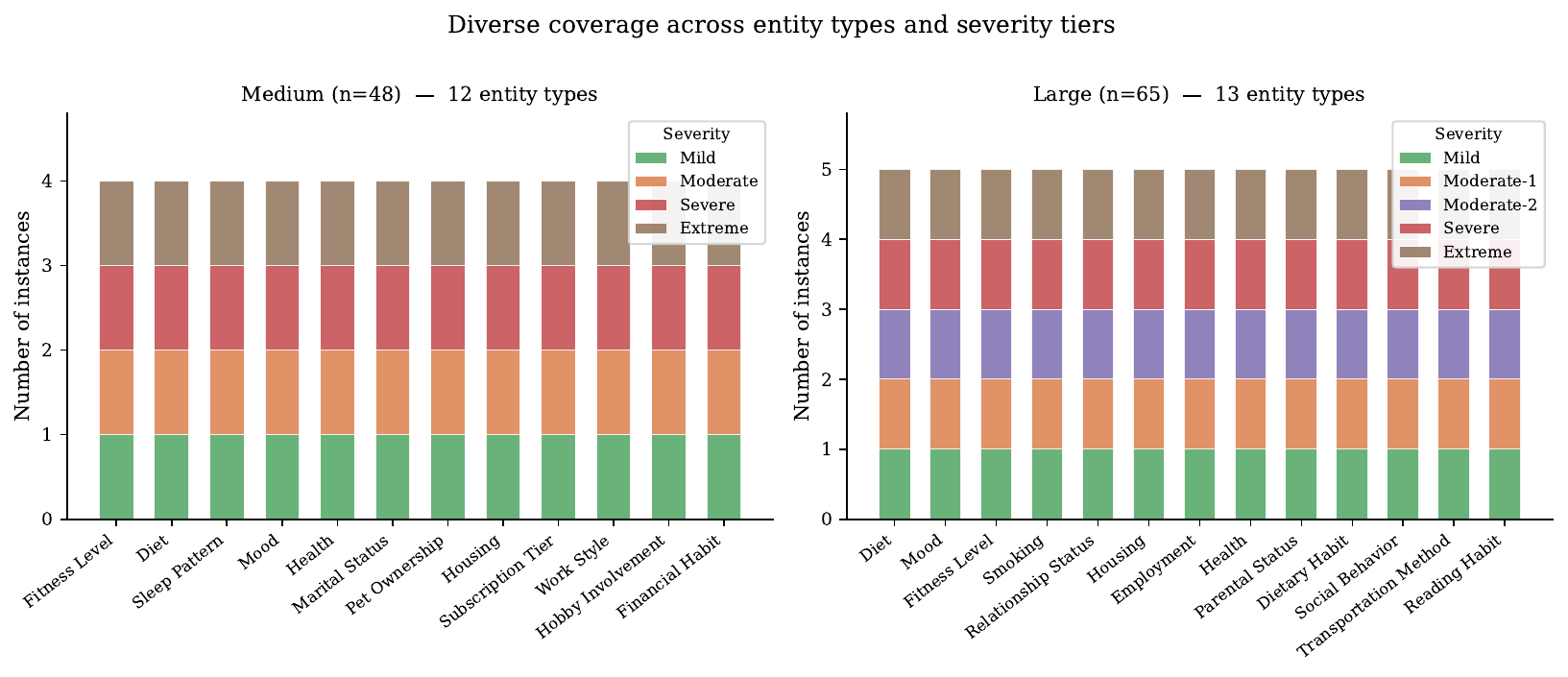}
\caption{Distribution of synthetic entity attributes (stacked by contradiction severity) in Chat-Medium ($n{=}48$) and Chat-Large ($n{=}65$).}
\label{app:fig:dataset_entity_distribution}
\end{figure}

\begin{figure}[htbp]
\centering
\includegraphics[width=\linewidth]{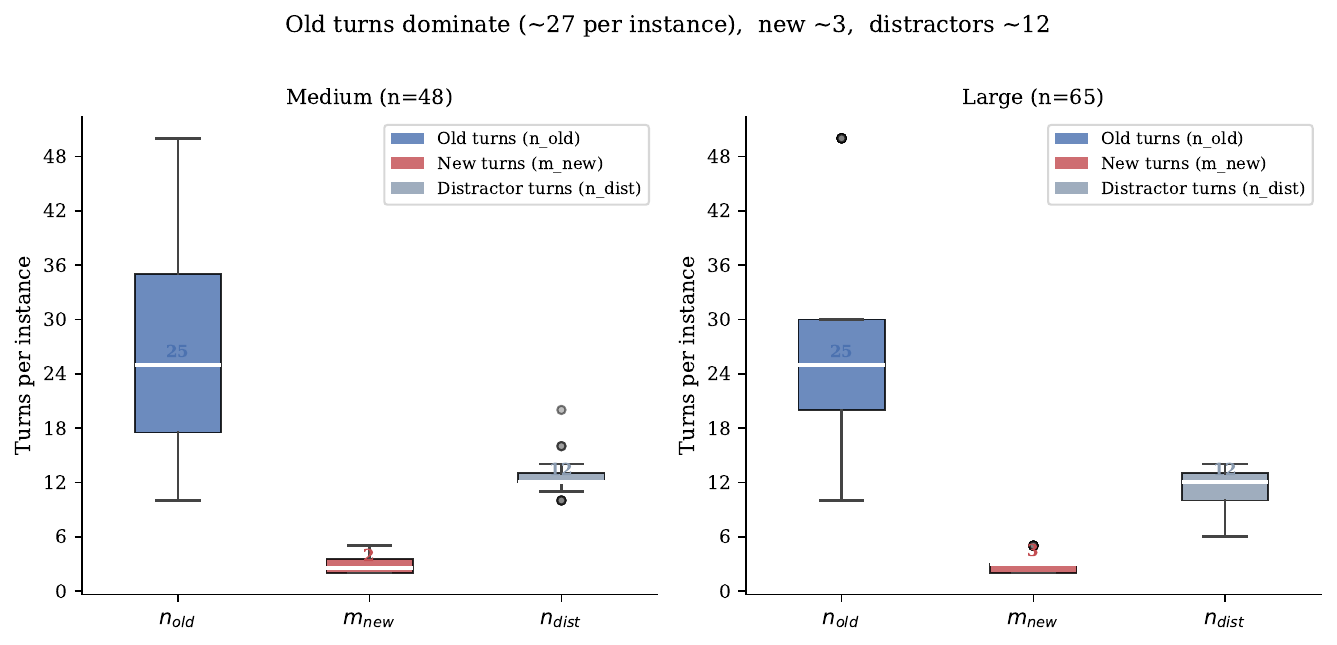}
\caption{Per-instance counts of chat turns by scenario role: \emph{old} (establishes superseded state), \emph{new} (current ground truth), and \emph{distractor} (neutral context).}
\label{app:fig:dataset_turn_composition}
\end{figure}

\begin{figure}[htbp]
\centering
\includegraphics[width=\linewidth]{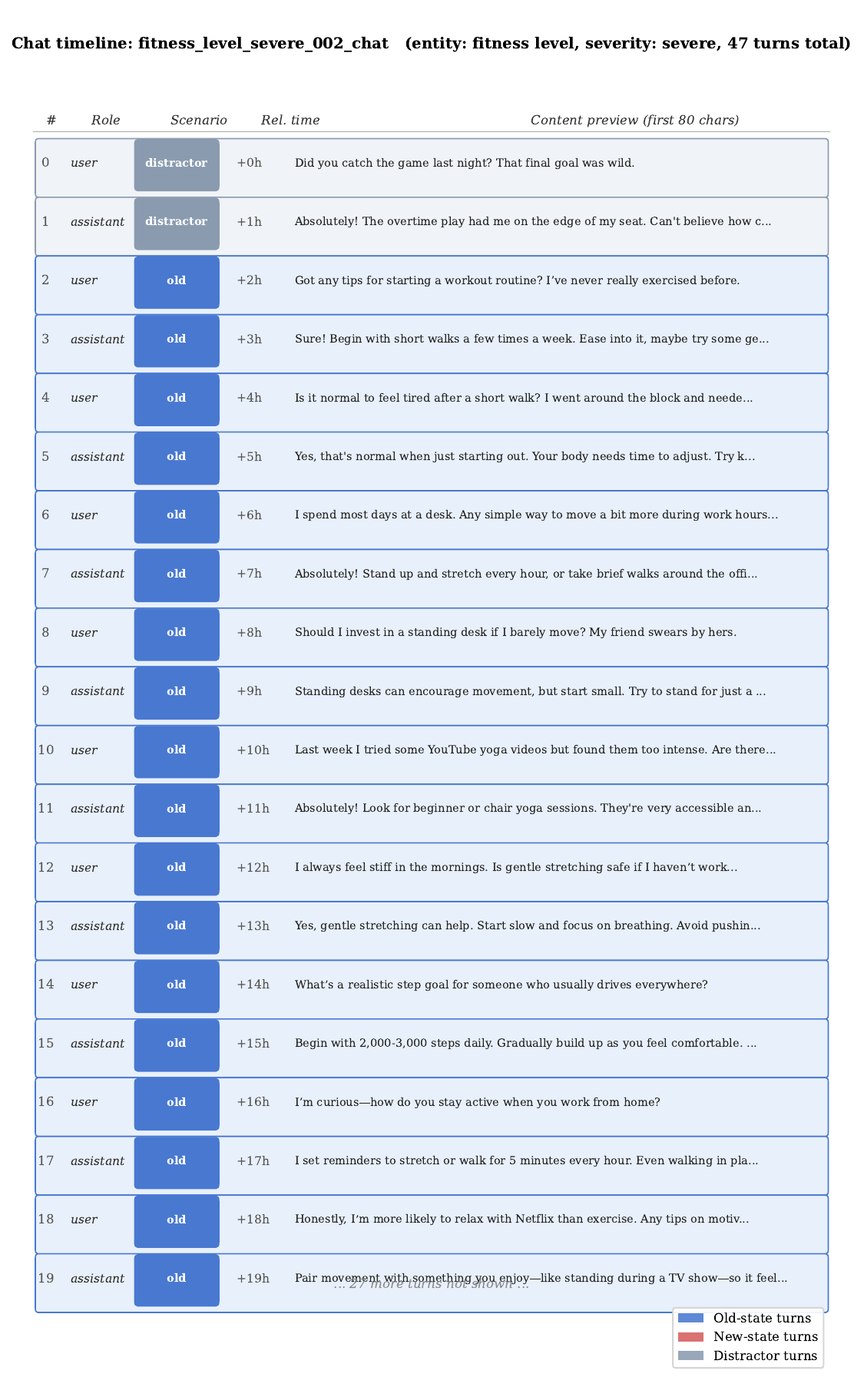}
\caption{Illustrative multi-turn instance: user/assistant turns are color-coded by role to show how state change is conveyed without explicit status keywords.}
\label{app:fig:dataset_chat_example}
\end{figure}


\subsection{Primary Benchmark Tables}
\label{sec:appendix_primary_tables}

\begin{table}[tb]
\centering\small

\begin{tabular}{llrrrrl}
\toprule
Dataset & STM & Std.\ RAG & Timestamp & GC-Mem & $\Delta$GC & Cells \\
\midrule
Medium & STM=off & 24.4\% & 24.4\% & \textbf{40.2\%} & +15.8p \\
Medium & STM=on & 46.1\% & 71.6\% & \textbf{68.1\%} & +22.0p \\
\midrule
Large & STM=off & 15.9\% & 16.2\% & \textbf{30.4\%} & +14.4p \\
Large & STM=on & 34.4\% & 67.5\% & \textbf{61.5\%} & +27.1p \\
\bottomrule
\end{tabular}
\caption{Headline accuracy (\%) averaged across models and prompt modes for each dataset×STM condition. Columns compare Standard RAG, Timestamp re-ranking, and GC-Mem; $\Delta$GC reports the GC-Mem improvement in percentage points. STM=on rows indicate the Short-Term Memory recency buffer was enabled.}
\label{app:tab:headline}
\end{table}

\begin{landscape}
\centering\scriptsize
\setlength{\tabcolsep}{3pt}
\begin{longtable}{llllrrrrrrrrrrr}
\caption{Per-condition accuracy breakdown across models, prompt modes, STM settings and datasets. Columns: Std = Standard RAG, TS = Timestamp re-rank baseline, GC = GC-Mem (best shown in bold). $\Delta$GC = GC-Mem $-$ Standard RAG (percentage points). Rec$_{\mathrm{old/new/dis}}$ = retrieval recall for old-state / new-state / distractor chunks. "Shadows" is mean number of chunks removed by the conflict oracle $\kappa$ per instance; "Err" lists API errors excluded from accuracy; "Sec" shows mean seconds per instance. This granular table is the source for the per-condition figures in the report.}\
\\\toprule
Dataset & STM & Model & Mode & Std & TS & GC & $\Delta$GC & Rec$_\mathrm{old}$ & Rec$_\mathrm{new}$ & Rec$_\mathrm{dis}$ & Shadows & $n$ & Err & Sec \\
\\\midrule
\\\endfirsthead
\\\multicolumn{15}{c}%
{\tablename\ \thetable\ -- \textit{Continued from previous page}} \\
\\\toprule
Dataset & STM & Model & Mode & Std & TS & GC & $\Delta$GC & Rec$_\mathrm{old}$ & Rec$_\mathrm{new}$ & Rec$_\mathrm{dis}$ & Shadows & $n$ & Err & Sec \\
\\\midrule
\\\endhead
\\\midrule \\\multicolumn{15}{r}{\textit{Continued on next page}} \\
\\\endfoot
\\\bottomrule
\\\endlastfoot
Medium & STM=off & DeepSeek-V3 & Forced-choice & 25.0\% & 20.8\% & \textbf{45.8\%} & +20.8p & 24.1\% & 37.2\% & 5.0\% & 1.73 & 48 & 0 & 267 \\
 & STM=off & DeepSeek-V3 & Allow-unclear & 23.4\% & 19.1\% & \textbf{40.4\%} & +17.0p & 24.3\% & 38.0\% & 5.0\% & 1.77 & 47 & 1 & 259 \\
 & STM=off & GPT-4o & Forced-choice & 20.8\% & 22.9\% & \textbf{47.9\%} & +27.1p & 24.1\% & 37.2\% & 5.0\% & 2.27 & 48 & 0 & 669 \\
 & STM=off & GPT-4o & Allow-unclear & 12.5\% & 18.8\% & \textbf{29.2\%} & +16.7p & 24.1\% & 37.2\% & 5.0\% & 2.19 & 48 & 0 & 486 \\
 & STM=off & GPT-4.1 & Forced-choice & 20.8\% & 22.9\% & \textbf{43.8\%} & +22.9p & 24.1\% & 37.2\% & 5.0\% & 1.73 & 48 & 0 & 387 \\
 & STM=off & GPT-4.1 & Allow-unclear & 22.9\% & 22.9\% & \textbf{35.4\%} & +12.5p & 24.1\% & 37.2\% & 5.0\% & 1.69 & 48 & 0 & 375 \\
 & STM=off & GPT-4.1-mini & Forced-choice & 33.3\% & 31.2\% & \textbf{33.3\%} & +0.0p & 24.1\% & 37.2\% & 5.0\% & 0.92 & 48 & 0 & 561 \\
 & STM=off & GPT-4.1-mini & Allow-unclear & 29.2\% & 27.1\% & \textbf{37.5\%} & +8.3p & 24.1\% & 37.2\% & 5.0\% & 0.88 & 48 & 0 & 560 \\
 & STM=off & GPT-5-Chat & Forced-choice & 34.0\% & 29.8\% & \textbf{48.9\%} & +14.9p & 23.8\% & 38.0\% & 5.1\% & 1.70 & 47 & 1 & 239 \\
 & STM=off & GPT-5-Chat & Allow-unclear & 23.4\% & 25.5\% & \textbf{40.4\%} & +17.0p & 24.2\% & 38.0\% & 4.9\% & 1.66 & 47 & 1 & 243 \\
 & STM=off & o4-mini & Forced-choice & 25.0\% & 25.0\% & \textbf{39.6\%} & +14.6p & 24.1\% & 37.2\% & 5.0\% & 0.92 & 48 & 0 & 1347 \\
 & STM=off & o4-mini & Allow-unclear & 22.9\% & 18.8\% & \textbf{33.3\%} & +10.4p & 24.1\% & 37.2\% & 5.0\% & 0.90 & 48 & 0 & 1330 \\
 & STM=off & Phi-4 & Forced-choice & 31.2\% & 35.4\% & \textbf{47.9\%} & +16.7p & 24.1\% & 37.2\% & 5.0\% & 1.67 & 48 & 0 & 330 \\
 & STM=off & Phi-4 & Allow-unclear & 16.7\% & 20.8\% & \textbf{39.6\%} & +22.9p & 24.1\% & 37.2\% & 5.0\% & 1.73 & 48 & 0 & 297 \\
 & STM=on & DeepSeek-V3 & Forced-choice & 37.5\% & 79.2\% & \textbf{70.8\%} & +33.3p & 32.1\% & 98.2\% & 43.3\% & 3.81 & 48 & 0 & 868 \\
 & STM=on & DeepSeek-V3 & Allow-unclear & 29.2\% & 62.5\% & \textbf{60.4\%} & +31.2p & 32.1\% & 98.2\% & 43.3\% & 3.77 & 48 & 0 & 914 \\
 & STM=on & GPT-4o & Forced-choice & 43.8\% & 79.2\% & \textbf{77.1\%} & +33.3p & 32.1\% & 98.2\% & 43.3\% & 4.77 & 48 & 0 & 1198 \\
 & STM=on & GPT-4o & Allow-unclear & 45.8\% & 62.5\% & \textbf{68.8\%} & +22.9p & 32.1\% & 98.2\% & 43.3\% & 4.71 & 48 & 0 & 1252 \\
 & STM=on & GPT-4.1 & Forced-choice & 46.8\% & 76.6\% & \textbf{70.2\%} & +23.4p & 32.2\% & 98.2\% & 43.0\% & 4.79 & 47 & 1 & 733 \\
 & STM=on & GPT-4.1 & Allow-unclear & 40.4\% & 72.3\% & \textbf{68.1\%} & +27.7p & 32.2\% & 98.2\% & 43.0\% & 4.91 & 47 & 1 & 720 \\
 & STM=on & GPT-4.1-mini & Forced-choice & 58.3\% & 72.9\% & \textbf{66.7\%} & +8.3p & 32.1\% & 98.2\% & 43.3\% & 1.90 & 48 & 0 & 1514 \\
 & STM=on & GPT-4.1-mini & Allow-unclear & 52.1\% & 72.9\% & \textbf{62.5\%} & +10.4p & 32.1\% & 98.2\% & 43.3\% & 2.17 & 48 & 0 & 1092 \\
 & STM=on & GPT-5-Chat & Forced-choice & 53.2\% & 72.3\% & \textbf{72.3\%} & +19.1p & 32.2\% & 98.2\% & 43.0\% & 3.94 & 47 & 1 & 1194 \\
 & STM=on & GPT-5-Chat & Allow-unclear & 53.2\% & 68.1\% & \textbf{70.2\%} & +17.0p & 32.2\% & 98.2\% & 43.0\% & 3.91 & 47 & 1 & 1060 \\
 & STM=on & o4-mini & Forced-choice & 45.8\% & 68.8\% & \textbf{58.3\%} & +12.5p & 32.1\% & 98.2\% & 43.3\% & 1.44 & 48 & 0 & 3113 \\
 & STM=on & o4-mini & Allow-unclear & 41.7\% & 62.5\% & \textbf{56.2\%} & +14.6p & 32.1\% & 98.2\% & 43.3\% & 1.50 & 48 & 0 & 2874 \\
 & STM=on & Phi-4 & Forced-choice & 58.3\% & 85.4\% & \textbf{83.3\%} & +25.0p & 32.1\% & 98.2\% & 43.3\% & 3.92 & 48 & 0 & 709 \\
 & STM=on & Phi-4 & Allow-unclear & 39.6\% & 66.7\% & \textbf{68.8\%} & +29.2p & 32.1\% & 98.2\% & 43.3\% & 3.71 & 48 & 0 & 623 \\
\midrule
Large & STM=off & DeepSeek-V3 & Forced-choice & 15.4\% & 13.9\% & \textbf{38.5\%} & +23.1p & 22.7\% & 29.9\% & 4.5\% & 1.54 & 65 & 0 & 366 \\
 & STM=off & DeepSeek-V3 & Allow-unclear & 15.5\% & 19.0\% & \textbf{31.0\%} & +15.5p & 22.9\% & 29.3\% & 4.6\% & 1.38 & 58 & 7 & 365 \\
 & STM=off & GPT-4o & Forced-choice & 15.4\% & 16.9\% & \textbf{35.4\%} & +20.0p & 22.7\% & 29.9\% & 4.5\% & 2.34 & 65 & 0 & 459 \\
 & STM=off & GPT-4o & Allow-unclear & 7.7\% & 7.7\% & \textbf{30.8\%} & +23.1p & 22.7\% & 29.9\% & 4.5\% & 2.32 & 65 & 0 & 441 \\
 & STM=off & GPT-4.1 & Forced-choice & 15.4\% & 16.9\% & \textbf{33.9\%} & +18.5p & 22.7\% & 29.9\% & 4.5\% & 1.95 & 65 & 0 & 480 \\
 & STM=off & GPT-4.1 & Allow-unclear & 16.9\% & 21.5\% & \textbf{29.2\%} & +12.3p & 22.7\% & 29.9\% & 4.5\% & 1.82 & 65 & 0 & 478 \\
 & STM=off & GPT-4.1-mini & Forced-choice & 18.5\% & 21.5\% & \textbf{27.7\%} & +9.2p & 22.7\% & 29.9\% & 4.5\% & 0.60 & 65 & 0 & 414 \\
 & STM=off & GPT-4.1-mini & Allow-unclear & 20.0\% & 18.5\% & \textbf{21.5\%} & +1.5p & 22.7\% & 29.9\% & 4.5\% & 0.62 & 65 & 0 & 407 \\
 & STM=off & GPT-5-Chat & Forced-choice & 20.0\% & 16.0\% & \textbf{36.0\%} & +16.0p & 22.4\% & 29.9\% & 5.8\% & 1.84 & 50 & 15 & 650 \\
 & STM=off & GPT-5-Chat & Allow-unclear & 20.0\% & 15.4\% & \textbf{35.4\%} & +15.4p & 22.7\% & 29.9\% & 4.5\% & 1.51 & 65 & 0 & 826 \\
 & STM=off & o4-mini & Forced-choice & 10.8\% & 13.9\% & \textbf{27.7\%} & +16.9p & 22.7\% & 29.9\% & 4.5\% & 0.55 & 65 & 0 & 1895 \\
 & STM=off & o4-mini & Allow-unclear & 12.3\% & 12.3\% & \textbf{23.1\%} & +10.8p & 22.7\% & 29.9\% & 4.5\% & 0.60 & 65 & 0 & 2139 \\
 & STM=off & Phi-4 & Forced-choice & 18.5\% & 20.0\% & \textbf{30.8\%} & +12.3p & 22.7\% & 29.9\% & 4.5\% & 1.15 & 65 & 0 & 345 \\
 & STM=off & Phi-4 & Allow-unclear & 16.9\% & 13.9\% & \textbf{24.6\%} & +7.7p & 22.7\% & 29.9\% & 4.5\% & 1.23 & 65 & 0 & 331 \\
 & STM=on & DeepSeek-V3 & Forced-choice & 27.7\% & 80.0\% & \textbf{69.2\%} & +41.5p & 29.1\% & 98.5\% & 47.8\% & 4.51 & 65 & 0 & 1270 \\
 & STM=on & DeepSeek-V3 & Allow-unclear & 23.1\% & 61.5\% & \textbf{63.1\%} & +40.0p & 29.1\% & 98.5\% & 47.8\% & 4.46 & 65 & 0 & 1207 \\
 & STM=on & GPT-4o & Forced-choice & 36.9\% & 80.0\% & \textbf{83.1\%} & +46.2p & 29.1\% & 98.5\% & 47.8\% & 5.51 & 65 & 0 & 1593 \\
 & STM=on & GPT-4o & Allow-unclear & 33.9\% & 47.7\% & \textbf{58.5\%} & +24.6p & 29.1\% & 98.5\% & 47.8\% & 5.43 & 65 & 0 & 1613 \\
 & STM=on & GPT-4.1 & Forced-choice & 28.1\% & 79.7\% & \textbf{67.2\%} & +39.1p & 29.3\% & 98.4\% & 48.2\% & 5.09 & 64 & 1 & 794 \\
 & STM=on & GPT-4.1 & Allow-unclear & 33.3\% & 55.6\% & \textbf{55.6\%} & +22.2p & 28.9\% & 98.4\% & 48.0\% & 4.87 & 63 & 2 & 791 \\
 & STM=on & GPT-4.1-mini & Forced-choice & 36.9\% & 75.4\% & \textbf{55.4\%} & +18.5p & 29.1\% & 98.5\% & 47.8\% & 2.29 & 65 & 0 & 1019 \\
 & STM=on & GPT-4.1-mini & Allow-unclear & 35.4\% & 60.0\% & \textbf{47.7\%} & +12.3p & 29.1\% & 98.5\% & 47.8\% & 2.20 & 65 & 0 & 946 \\
 & STM=on & GPT-5-Chat & Forced-choice & 39.7\% & 76.2\% & \textbf{68.2\%} & +28.6p & 28.9\% & 98.4\% & 48.0\% & 4.25 & 63 & 2 & 1166 \\
 & STM=on & GPT-5-Chat & Allow-unclear & 39.1\% & 68.8\% & \textbf{65.6\%} & +26.6p & 29.3\% & 98.4\% & 48.2\% & 4.31 & 64 & 1 & 982 \\
 & STM=on & o4-mini & Forced-choice & 33.9\% & 70.8\% & \textbf{44.6\%} & +10.8p & 29.1\% & 98.5\% & 47.8\% & 1.62 & 65 & 0 & 3954 \\
 & STM=on & o4-mini & Allow-unclear & 32.3\% & 58.5\% & \textbf{46.2\%} & +13.9p & 29.1\% & 98.5\% & 47.8\% & 1.54 & 65 & 0 & 4030 \\
 & STM=on & Phi-4 & Forced-choice & 47.7\% & 80.0\% & \textbf{75.4\%} & +27.7p & 29.1\% & 98.5\% & 47.8\% & 4.28 & 65 & 0 & 762 \\
 & STM=on & Phi-4 & Allow-unclear & 33.9\% & 50.8\% & \textbf{61.5\%} & +27.7p & 29.1\% & 98.5\% & 47.8\% & 4.42 & 65 & 0 & 792 \\
\bottomrule
\end{longtable}
\label{app:tab:per_condition}
\end{landscape}

\begin{table}[tb]
\centering\small

\begin{tabular}{llrr}
\toprule
STM & Model & Forced-choice $\Delta$GC & Allow-unclear $\Delta$GC \\
\midrule
STM=off & DeepSeek-V3 & \textbf{+22.0p} & +16.3p \\
STM=off & GPT-4o & \textbf{+23.5p} & +19.9p \\
STM=off & GPT-4.1 & \textbf{+20.7p} & +12.4p \\
STM=off & GPT-4.1-mini & +4.6p & \textbf{+4.9p} \\
STM=off & GPT-5-Chat & +15.4p & \textbf{+16.2p} \\
STM=off & o4-mini & \textbf{+15.8p} & +10.6p \\
STM=off & Phi-4 & +14.5p & \textbf{+15.3p} \\
\midrule
STM=on & DeepSeek-V3 & \textbf{+37.4p} & +35.6p \\
STM=on & GPT-4o & \textbf{+39.7p} & +23.8p \\
STM=on & GPT-4.1 & \textbf{+31.2p} & +24.9p \\
STM=on & GPT-4.1-mini & \textbf{+13.4p} & +11.4p \\
STM=on & GPT-5-Chat & \textbf{+23.9p} & +21.8p \\
STM=on & o4-mini & +11.6p & \textbf{+14.2p} \\
STM=on & Phi-4 & +26.3p & \textbf{+28.4p} \\
\bottomrule
\end{tabular}
\caption{Prompt-mode ablation: mean $\Delta$GC (pp) across both datasets for Forced-choice vs.\ Allow-unclear prompts, per model and STM setting. Bold indicates the stronger prompt mode. GC-Mem improves over Standard RAG in all 12 cells.}
\label{app:tab:mode_ablation}
\end{table}

\begin{table}[tb]
\centering\small

\begin{tabular}{llrrrr}
\toprule
Dataset & STM & Recall$_\mathrm{old}$ & Recall$_\mathrm{new}$ & Recall$_\mathrm{dis}$ & Cells \\
\midrule
Medium & STM=off & 24.1\% & \textbf{37.4\%} & 5.0\% & 14/6 \\
Medium & STM=on & 32.1\% & \textbf{98.2\%} & 43.2\% & 14/6 \\
\midrule
Large & STM=off & 22.7\% & \textbf{29.9\%} & 4.6\% & 14/6 \\
Large & STM=on & 29.1\% & \textbf{98.5\%} & 47.9\% & 14/6 \\
\bottomrule
\end{tabular}
\caption{Retrieval recall before and after activating the Short-Term Memory (STM) buffer, averaged over 3 models $\times$ 2 prompt modes. Recall$_{\mathrm{new}}$ = fraction of updated (new-truth) chunks retrieved; Recall$_{\mathrm{dis}}$ = fraction of distractor chunks retrieved (lower is better). STM raises Recall$_{\mathrm{new}}$ from $\sim$37\% to $\sim$98\%.}
\label{app:tab:stm_recall}
\end{table}

\begin{table}[tb]
\centering\small

\begin{tabular}{lllrr}
\toprule
STM & Model & Mode & Avg.\ Shadows Removed & Datasets \\
\midrule
STM=off & DeepSeek-V3 & Forced-choice & 1.64 & 2/2 \\
STM=off & DeepSeek-V3 & Allow-unclear & 1.57 & 2/2 \\
STM=off & GPT-4o & Forced-choice & 2.30 & 2/2 \\
STM=off & GPT-4o & Allow-unclear & 2.25 & 2/2 \\
STM=off & GPT-4.1 & Forced-choice & 1.84 & 2/2 \\
STM=off & GPT-4.1 & Allow-unclear & 1.75 & 2/2 \\
STM=off & GPT-4.1-mini & Forced-choice & 0.76 & 2/2 \\
STM=off & GPT-4.1-mini & Allow-unclear & 0.75 & 2/2 \\
STM=off & GPT-5-Chat & Forced-choice & 1.77 & 2/2 \\
STM=off & GPT-5-Chat & Allow-unclear & 1.58 & 2/2 \\
STM=off & o4-mini & Forced-choice & 0.74 & 2/2 \\
STM=off & o4-mini & Allow-unclear & 0.75 & 2/2 \\
STM=off & Phi-4 & Forced-choice & 1.41 & 2/2 \\
STM=off & Phi-4 & Allow-unclear & 1.48 & 2/2 \\
\midrule
STM=on & DeepSeek-V3 & Forced-choice & 4.16 & 2/2 \\
STM=on & DeepSeek-V3 & Allow-unclear & 4.12 & 2/2 \\
STM=on & GPT-4o & Forced-choice & 5.14 & 2/2 \\
STM=on & GPT-4o & Allow-unclear & 5.07 & 2/2 \\
STM=on & GPT-4.1 & Forced-choice & 4.94 & 2/2 \\
STM=on & GPT-4.1 & Allow-unclear & 4.89 & 2/2 \\
STM=on & GPT-4.1-mini & Forced-choice & 2.09 & 2/2 \\
STM=on & GPT-4.1-mini & Allow-unclear & 2.19 & 2/2 \\
STM=on & GPT-5-Chat & Forced-choice & 4.09 & 2/2 \\
STM=on & GPT-5-Chat & Allow-unclear & 4.11 & 2/2 \\
STM=on & o4-mini & Forced-choice & 1.53 & 2/2 \\
STM=on & o4-mini & Allow-unclear & 1.52 & 2/2 \\
STM=on & Phi-4 & Forced-choice & 4.10 & 2/2 \\
STM=on & Phi-4 & Allow-unclear & 4.06 & 2/2 \\
\bottomrule
\end{tabular}
\caption{Mean number of shadow chunks removed per instance by the GC-Mem operator $\kappa$, averaged over both datasets. STM=on substantially increases the number of outdated chunks surfaced and pruned, confirming that STM provides richer context for the conflict detector.}
\label{app:tab:kappa_shadows}
\end{table}

\begin{table}[tb]
\centering\small

\begin{tabular}{lllrrrll}
\toprule
Dataset & Model & Mode & Std.\ RAG & Timestamp & GC-Mem & Winner & Gap \\
\midrule
Medium & DeepSeek-V3 & Forced-choice & 37.5\% & \textbf{79.2\%} & 70.8\% & Timestamp & +8.3pp \\
 & DeepSeek-V3 & Allow-unclear & 29.2\% & \textbf{62.5\%} & 60.4\% & Timestamp & +2.1pp \\
 & GPT-4o & Forced-choice & 43.8\% & \textbf{79.2\%} & 77.1\% & Timestamp & +2.1pp \\
 & GPT-4o & Allow-unclear & 45.8\% & 62.5\% & \textbf{68.8\%} & GC-Mem & +6.2pp \\
 & GPT-4.1 & Forced-choice & 46.8\% & \textbf{76.6\%} & 70.2\% & Timestamp & +6.4pp \\
 & GPT-4.1 & Allow-unclear & 40.4\% & \textbf{72.3\%} & 68.1\% & Timestamp & +4.3pp \\
 & GPT-4.1-mini & Forced-choice & 58.3\% & \textbf{72.9\%} & 66.7\% & Timestamp & +6.2pp \\
 & GPT-4.1-mini & Allow-unclear & 52.1\% & \textbf{72.9\%} & 62.5\% & Timestamp & +10.4pp \\
 & GPT-5-Chat & Forced-choice & 53.2\% & 72.3\% & 72.3\% & Tie & 0.0pp \\
 & GPT-5-Chat & Allow-unclear & 53.2\% & 68.1\% & \textbf{70.2\%} & GC-Mem & +2.1pp \\
 & o4-mini & Forced-choice & 45.8\% & \textbf{68.8\%} & 58.3\% & Timestamp & +10.4pp \\
 & o4-mini & Allow-unclear & 41.7\% & \textbf{62.5\%} & 56.2\% & Timestamp & +6.2pp \\
 & Phi-4 & Forced-choice & 58.3\% & \textbf{85.4\%} & 83.3\% & Timestamp & +2.1pp \\
 & Phi-4 & Allow-unclear & 39.6\% & 66.7\% & \textbf{68.8\%} & GC-Mem & +2.1pp \\
\midrule
Large & DeepSeek-V3 & Forced-choice & 27.7\% & \textbf{80.0\%} & 69.2\% & Timestamp & +10.8pp \\
 & DeepSeek-V3 & Allow-unclear & 23.1\% & 61.5\% & \textbf{63.1\%} & GC-Mem & +1.5pp \\
 & GPT-4o & Forced-choice & 36.9\% & 80.0\% & \textbf{83.1\%} & GC-Mem & +3.1pp \\
 & GPT-4o & Allow-unclear & 33.9\% & 47.7\% & \textbf{58.5\%} & GC-Mem & +10.8pp \\
 & GPT-4.1 & Forced-choice & 28.1\% & \textbf{79.7\%} & 67.2\% & Timestamp & +12.5pp \\
 & GPT-4.1 & Allow-unclear & 33.3\% & 55.6\% & 55.6\% & Tie & 0.0pp \\
 & GPT-4.1-mini & Forced-choice & 36.9\% & \textbf{75.4\%} & 55.4\% & Timestamp & +20.0pp \\
 & GPT-4.1-mini & Allow-unclear & 35.4\% & \textbf{60.0\%} & 47.7\% & Timestamp & +12.3pp \\
 & GPT-5-Chat & Forced-choice & 39.7\% & \textbf{76.2\%} & 68.2\% & Timestamp & +7.9pp \\
 & GPT-5-Chat & Allow-unclear & 39.1\% & \textbf{68.8\%} & 65.6\% & Timestamp & +3.1pp \\
 & o4-mini & Forced-choice & 33.9\% & \textbf{70.8\%} & 44.6\% & Timestamp & +26.2pp \\
 & o4-mini & Allow-unclear & 32.3\% & \textbf{58.5\%} & 46.2\% & Timestamp & +12.3pp \\
 & Phi-4 & Forced-choice & 47.7\% & \textbf{80.0\%} & 75.4\% & Timestamp & +4.6pp \\
 & Phi-4 & Allow-unclear & 33.9\% & 50.8\% & \textbf{61.5\%} & GC-Mem & +10.8pp \\
\bottomrule
\end{tabular}
\caption{Head-to-head comparison of Timestamp re-ranking and GC-Mem under STM=on (where the STM buffer dramatically improves retrieval recall). Bold marks the stronger of the two non-trivial baselines. Timestamp becomes competitive under STM because high recall reduces its recency-ranking advantage gap, yet GC-Mem retains an edge in Allow-unclear (lenient) conditions.}
\label{app:tab:ts_race_stm}
\end{table}

\begin{table}[tb]
\centering\small

\begin{tabular}{llrrrr}
\toprule
Dataset & STM & Mean Wall Time (s) & Per-instance (s) & Avg.\ Shadows & Cells \\
\midrule
Medium & STM=off & 525 & 10.9 & 1.55 & 14/6 \\
Medium & STM=on & 1276 & 26.6 & 3.52 & 14/6 \\
\midrule
Large & STM=off & 685 & 10.5 & 1.39 & 14/6 \\
Large & STM=on & 1494 & 23.0 & 3.91 & 14/6 \\
\bottomrule
\end{tabular}
\caption{Wall-clock time per condition averaged over 3 models $\times$ 2 prompt modes, and mean shadows removed per instance. STM=on increases latency due to additional STM retrieval and conflict detection, but also removes substantially more shadow chunks.}
\label{app:tab:latency}
\end{table}

\subsection{Primary Benchmark Figures}
\label{sec:appendix_primary_figures}
\begin{figure}[htbp]
\centering
\includegraphics[width=\linewidth]{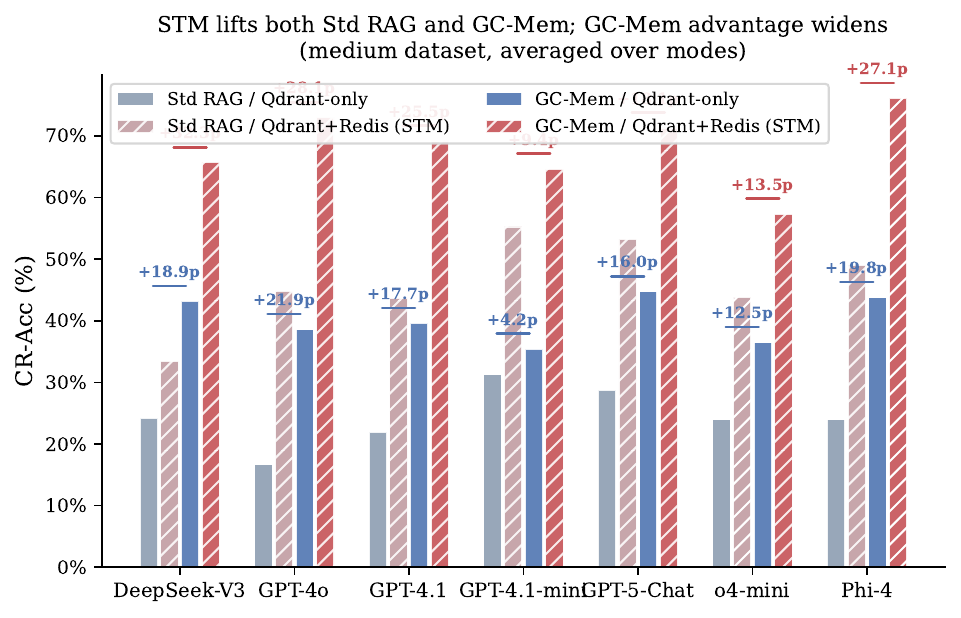}
\caption{Effect of enabling STM (Redis FIFO of recent chunks) on standard RAG vs.\ GC-Mem accuracy. STM raises both baselines by improving exposure of new-state evidence; GC-Mem retains a larger gap because it removes contradictions that simple recency heuristics leave in context.}
\label{app:fig:stm_lift}
\end{figure}

\FloatBarrier
\subsection{Example Prompt}
\label{app:example_prompt}
\begin{figure}[htbp]
\centering
\includegraphics[width=\linewidth]{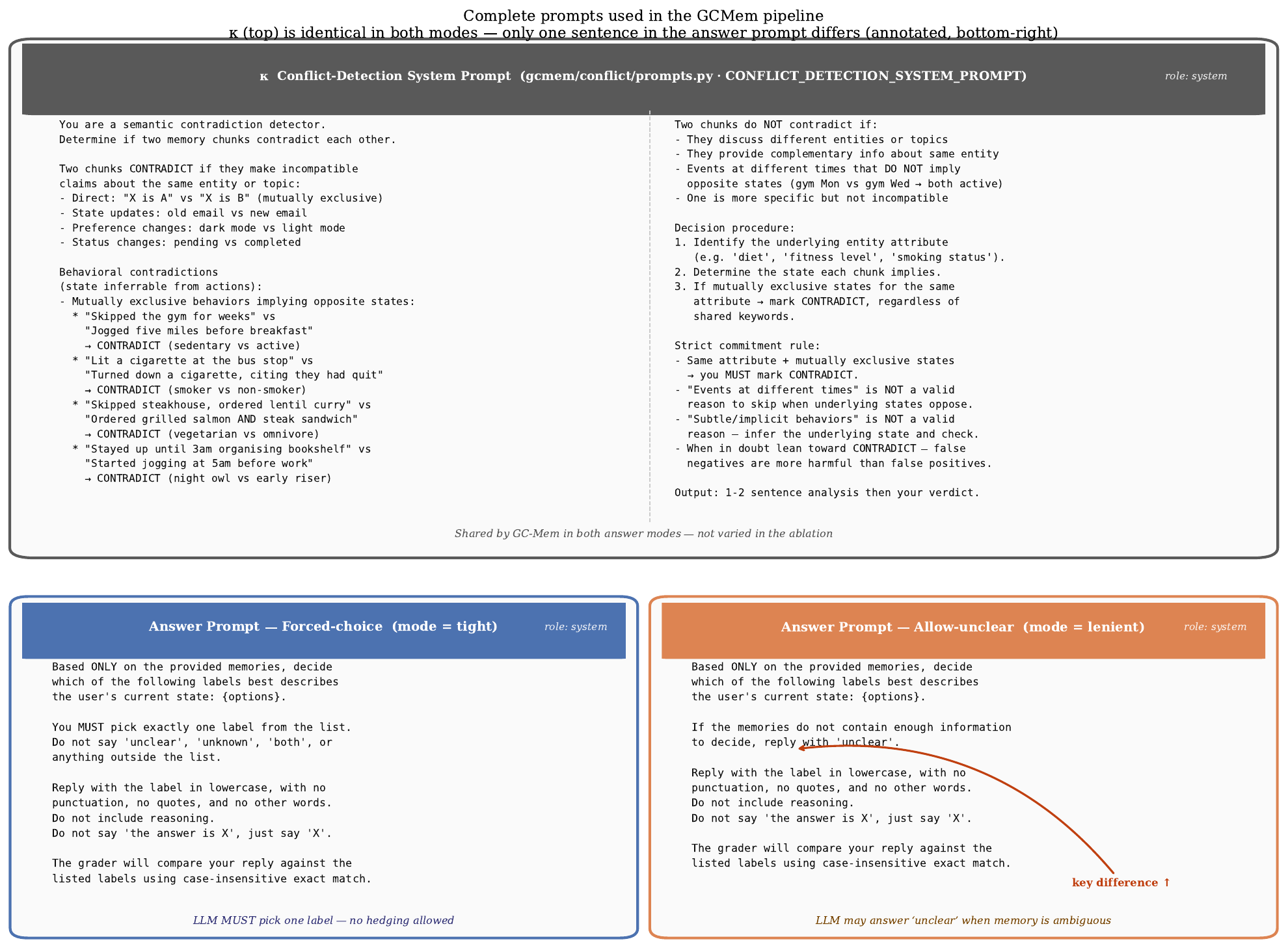}
\caption{Example prompt templates used for the prompt-mode ablation. The figure shows the two prompt styles compared: Forced-choice (must pick a label) and Allow-unclear (model may respond with ``unclear''). These were used across all models to measure robustness of GC-Mem to prompt framing.}
\label{app:fig:example_prompt}
\end{figure}

\clearpage
\begin{figure}[p]
\centering
\includegraphics[width=\linewidth,height=0.78\textheight,keepaspectratio]{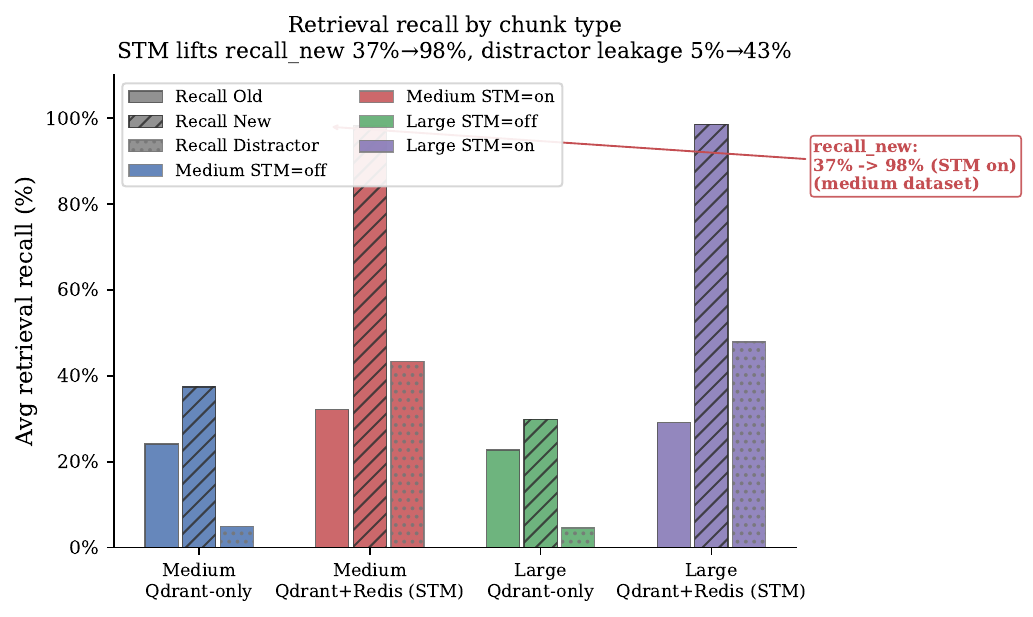}
\caption{Retrieval recall by chunk type (old-state, new-state, distractor) comparing STM off vs on. STM substantially raises recall of new-state chunks, which increases the number of contradictory pairs in the candidate set and therefore the benefit of applying the temporal dominance operator $\Phi_{\mathcal{T}}$.}
\label{app:fig:retrieval_recall}
\end{figure}

\begin{figure}[htbp]
\centering
\includegraphics[width=\linewidth]{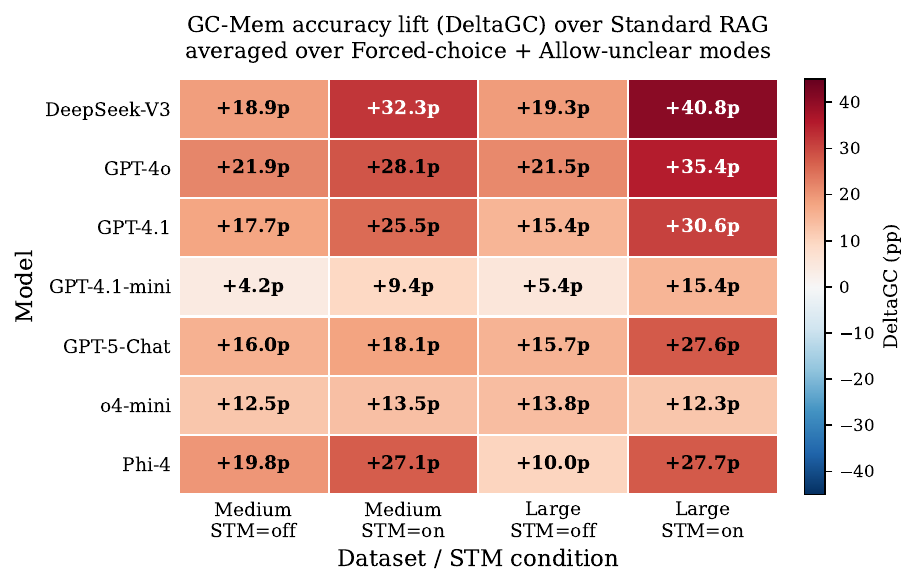}
\caption{Heatmap summarizing $\Delta$GC (GC-Mem minus Standard RAG) across model × dataset × STM cells. Warm colors indicate larger gains from GC-Mem; the heatmap serves as the single-glance summary of experimental effectiveness reported in the artifact.}
\label{app:fig:heatmap_dgc}
\end{figure}

\begin{figure}[htbp]
\centering
\includegraphics[width=\linewidth,height=0.65\textheight,keepaspectratio]{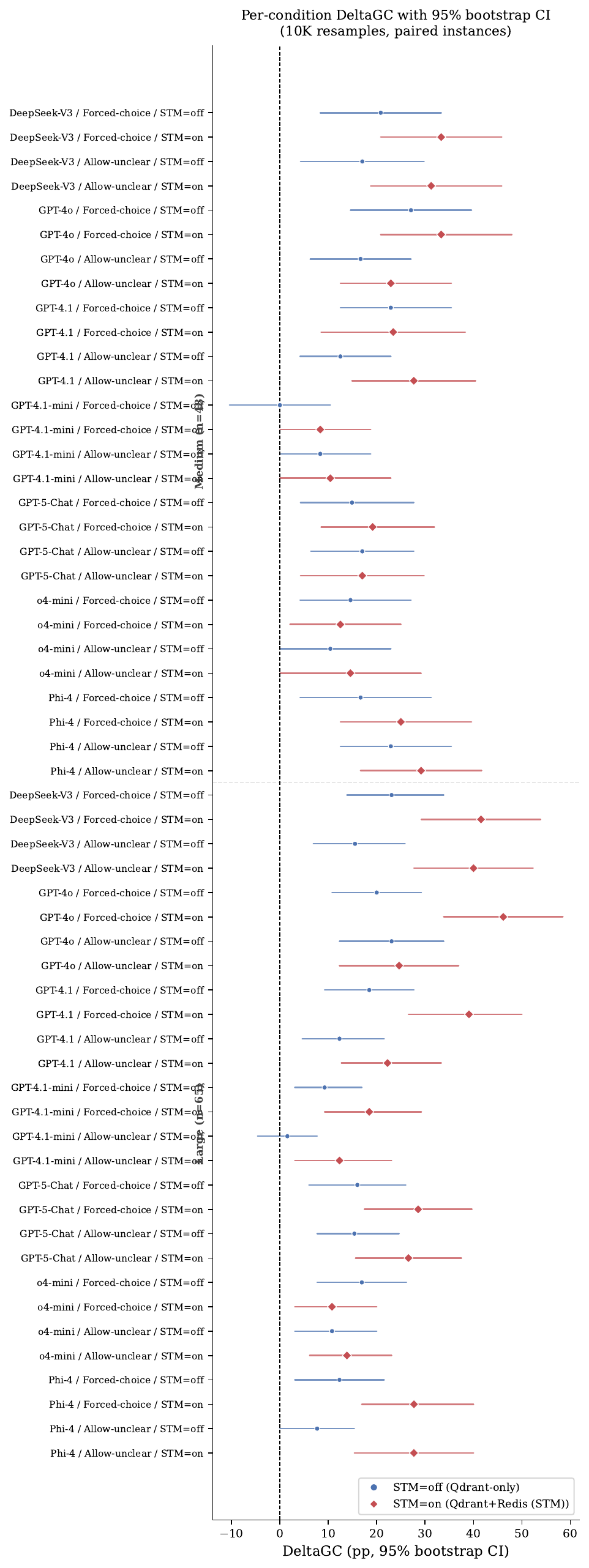}
\caption{Forest plot of paired bootstrap 95\% confidence intervals for $\Delta$GC (10{,}000 resamples per cell). Intervals entirely above zero support a consistent gain; intervals crossing zero flag weak $\kappa$ on that model or incomplete cells.}
\label{app:fig:forest_ci}
\end{figure}

\FloatBarrier
\begin{figure}[htbp]
\centering
\includegraphics[width=\linewidth]{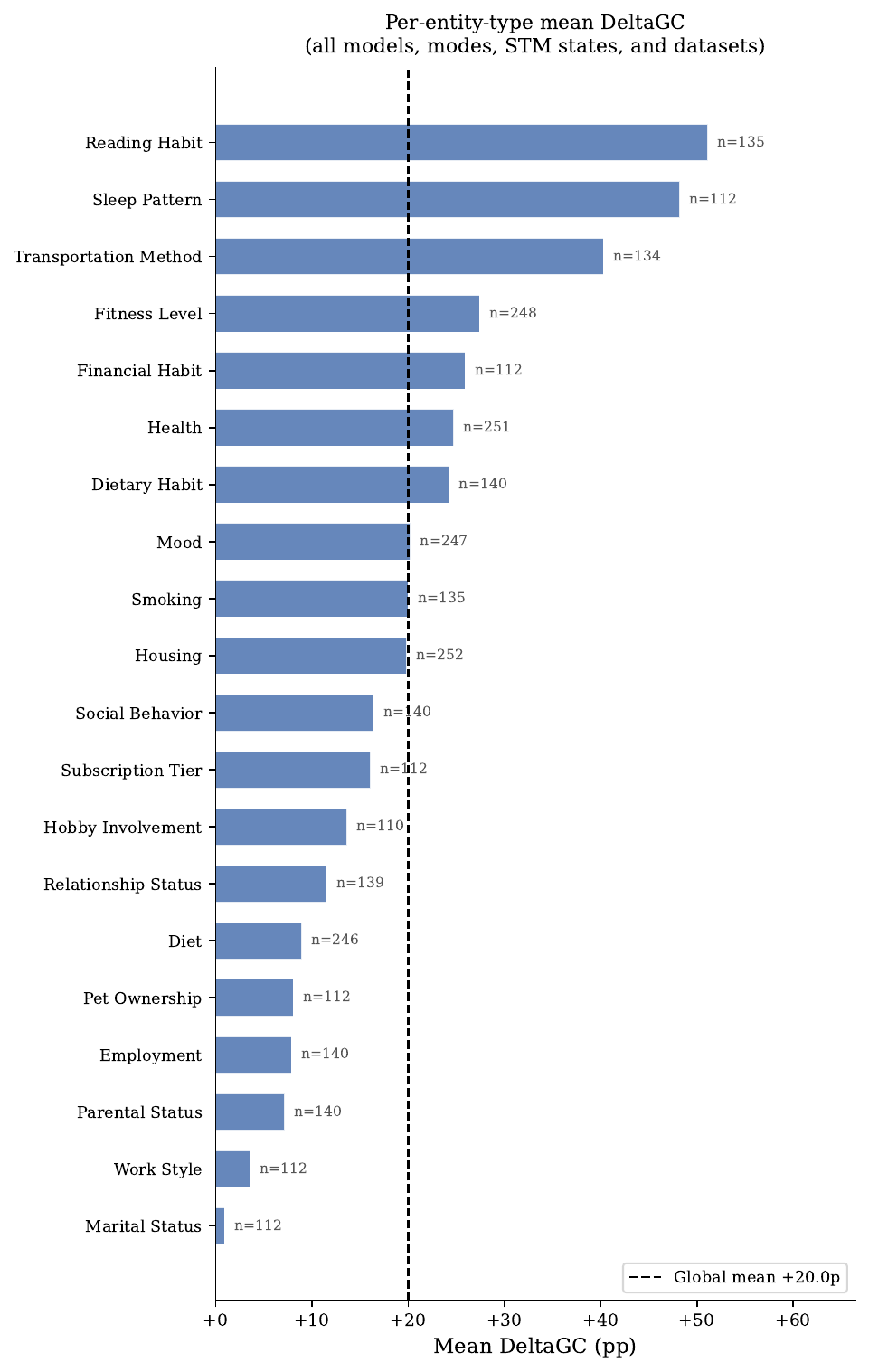}
\caption{$\Delta$GC broken down by synthetic entity type. Positive bars across categories indicate that GC-Mem is not driven by a single attribute family.}
\label{app:fig:per_entity_type}
\end{figure}

\subsection{Bag-of-Facts (No-Retrieval) Ablation}
\label{sec:appendix_bof}

\begin{table}[tb]
\centering\small

\begin{tabular}{lrrrrrrr}
\toprule
Dataset & $n$ & Mean Std & $\sigma$ & Mean TS & Mean GC & Mean $\Delta$GC & $\sigma$($\Delta$GC) \\
\midrule
Medium & 48 & 20.4\% & $\pm$10.7 & 18.5\% & \textbf{55.2\%} & +34.8p & $\pm$12.2 \\
Large & 65 & 12.1\% & $\pm$4.2 & 11.7\% & \textbf{36.5\%} & +24.4p & $\pm$11.6 \\
\bottomrule
\end{tabular}
\caption{Bag-of-facts (k=$\infty$) headline results: mean accuracy and $\Delta$GC across 7 models. Std = Standard RAG (all facts in context without GC), TS = Timestamp re-ranking, GC = GC-Mem. $\sigma$ is sample standard deviation across 7 models.}
\label{app:tab:bof_headline}
\end{table}

\begin{table}[tb]
\centering\small

\begin{tabular}{llrrrrrr}
\toprule
Model & Dataset & Std.\ RAG & Timestamp & GC-Mem & $\Delta$GC & Shadows & $n$ \\
\midrule
DeepSeek-V3 & Medium & 14.6\% & 12.5\% & \textbf{56.2\%} & +41.7p & 13.58 & 48 \\
 & Large & 9.2\% & 12.3\% & \textbf{52.3\%} & +43.1p & 12.98 & 65 \\
\midrule
GPT-4o & Medium & 16.7\% & 18.8\% & \textbf{60.4\%} & +43.8p & 16.67 & 48 \\
 & Large & 9.2\% & 6.2\% & \textbf{46.2\%} & +36.9p & 14.72 & 65 \\
\midrule
GPT-4.1 & Medium & 21.7\% & 15.2\% & \textbf{60.9\%} & +39.1p & 13.02 & 48 \\
 & Large & 14.1\% & 9.4\% & \textbf{31.2\%} & +17.2p & 10.47 & 65 \\
\midrule
GPT-4.1-mini & Medium & 14.6\% & 18.8\% & \textbf{37.5\%} & +22.9p & 7.42 & 48 \\
 & Large & 7.7\% & 7.7\% & \textbf{24.6\%} & +16.9p & 5.14 & 65 \\
\midrule
GPT-5-Chat & Medium & 39.1\% & 21.7\% & \textbf{71.7\%} & +32.6p & 13.39 & 48 \\
 & Large & 17.5\% & 11.1\% & \textbf{44.4\%} & +27.0p & 10.52 & 65 \\
\midrule
o4-mini & Medium & 29.2\% & 31.2\% & \textbf{43.8\%} & +14.6p & 7.25 & 48 \\
 & Large & 9.2\% & 15.4\% & \textbf{23.1\%} & +13.9p & 5.17 & 65 \\
\midrule
Phi-4 & Medium & 7.0\% & 11.6\% & \textbf{55.8\%} & +48.8p & 14.02 & 48 \\
 & Large & 17.6\% & 19.6\% & \textbf{33.3\%} & +15.7p & 13.45 & 65 \\
\bottomrule
\end{tabular}
\caption{Per-model bag-of-facts results across both dataset sizes. Std = Standard RAG accuracy (all facts provided); GC-Mem = accuracy after GC operator prunes outdated shadows (bold). $\Delta$GC = GC-Mem $-$ Std.~RAG in percentage points. Shadows = mean outdated chunks removed per instance by $\kappa$.}
\label{app:tab:bof_per_model}
\end{table}

\begin{table}[tb]
\centering\small

\begin{tabular}{rlrrrr}
\toprule
Rank & Model & Mean Std.\ RAG & Mean GC-Mem & Mean $\Delta$GC & Mean Shadows \\
\midrule
1 & DeepSeek-V3 & 11.9\% & \textbf{54.3\%} & +42.4p & 13.28 \\
2 & GPT-4o & 12.9\% & \textbf{53.3\%} & +40.3p & 15.70 \\
3 & Phi-4 & 12.3\% & \textbf{44.6\%} & +32.3p & 13.73 \\
4 & GPT-5-Chat & 28.3\% & \textbf{58.1\%} & +29.8p & 11.96 \\
5 & GPT-4.1 & 17.9\% & \textbf{46.1\%} & +28.2p & 11.75 \\
6 & GPT-4.1-mini & 11.1\% & \textbf{31.1\%} & +19.9p & 6.28 \\
7 & o4-mini & 19.2\% & \textbf{33.4\%} & +14.2p & 6.21 \\
\bottomrule
\end{tabular}
\caption{Models ranked by pooled mean $\Delta$GC (average of medium and large datasets). Mean Std.~RAG and Mean GC-Mem are averages of the two dataset sizes. Models with lower baseline accuracy (Std) tend to show larger GC-Mem gains, indicating that GC pruning is most effective when the unfiltered context is most misleading.}
\label{app:tab:bof_model_ranking}
\end{table}

\begin{table}[tb]
\centering\small

\begin{tabular}{lrrr}
\toprule
Model & Chat-RAG no-STM & Chat-RAG STM-on & Bag-of-facts (k=$\infty$) \\
\midrule
DeepSeek-V3 & +18.9p & +32.3p & \textbf{+41.7p} \\
GPT-5-Chat & +16.0p & +18.1p & \textbf{+32.6p} \\
Phi-4 & +19.8p & +27.1p & \textbf{+48.8p} \\
\bottomrule
\end{tabular}
\caption{$\Delta$GC (pp) on the medium dataset under three setups for the three models common to all experiments. Chat-RAG values are averaged over Forced-choice and Allow-unclear prompt modes. Bag-of-facts removes the Qdrant retrieval bottleneck (k=$\infty$, all facts in context) before GC-Mem is applied. The monotone increase Chat-RAG no-STM $<$ Chat-RAG STM-on $<$ Bag-of-facts confirms that \emph{GC-Mem advantage scales with retrieval quality}: when the context contains more outdated facts, GC pruning delivers a larger lift.}
\label{app:tab:bof_compare_chat}
\end{table}

\begin{figure}[H]
\centering
\includegraphics[width=\linewidth]{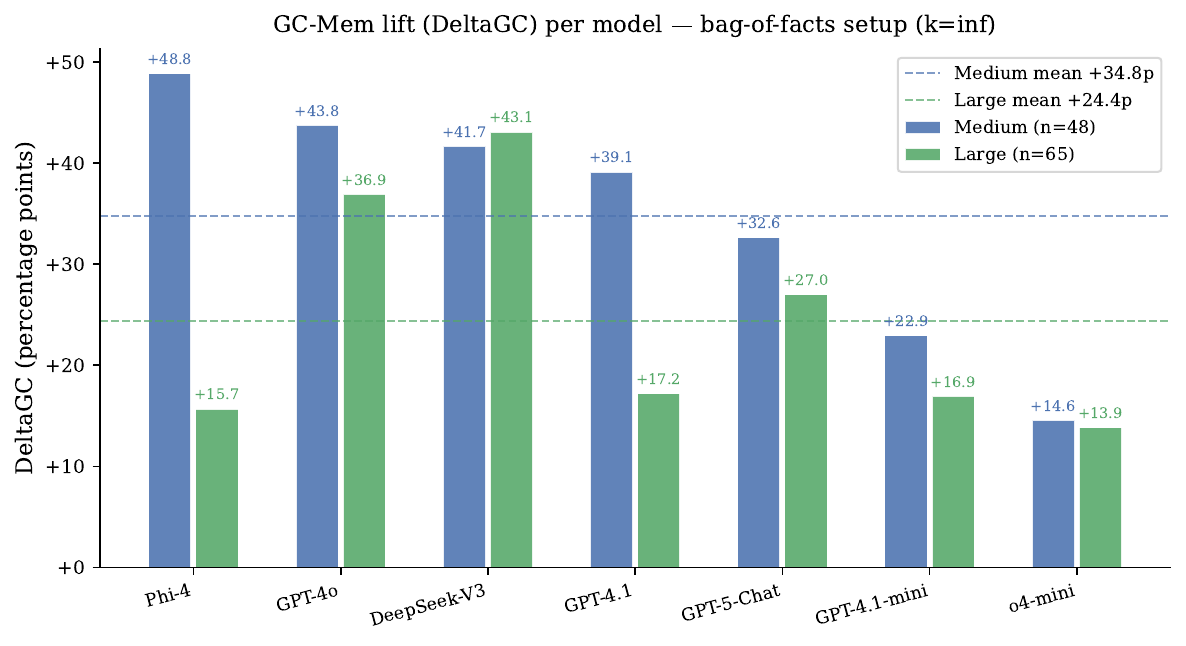}
\caption{Bag-of-facts (no retrieval) setting: $\Delta$GC by model when all old and new memories are placed directly in context ($k{\rightarrow}\infty$). This isolates the value of $\Phi_{\mathcal{T}}$ from retrieval noise.}
\label{fig:bof_per_model_dgc}
\end{figure}

\begin{figure}[H]
\centering
\includegraphics[width=\linewidth]{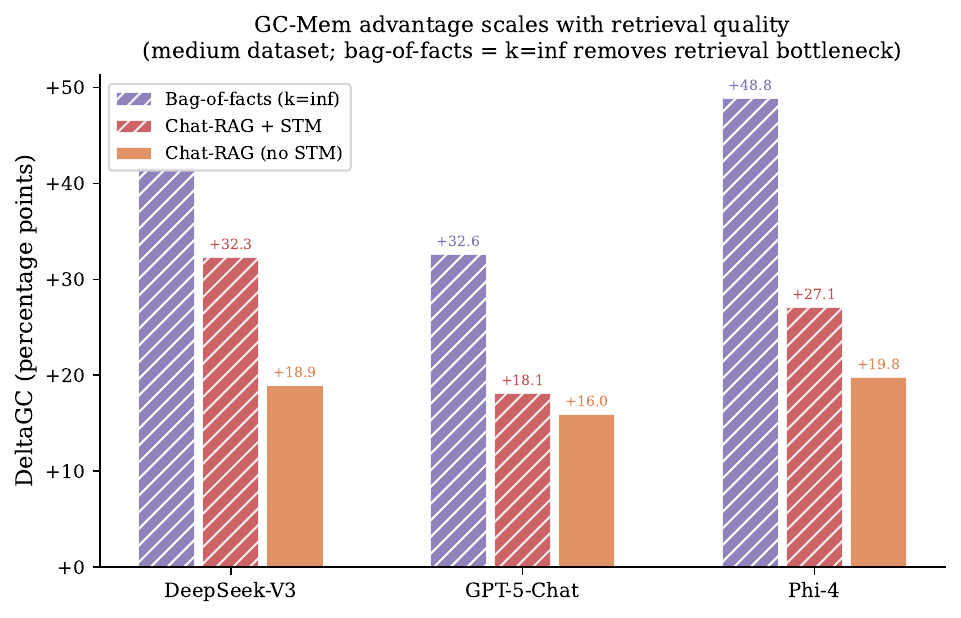}
\caption{Comparison of mean $\Delta$GC across chat-RAG (with and without STM) and bag-of-facts. Larger candidate sets with more conflicting statements yield larger GC-Mem lifts.}
\label{fig:bof_compare_setups}
\end{figure}

\begin{figure}[H]
\centering
\includegraphics[width=\linewidth]{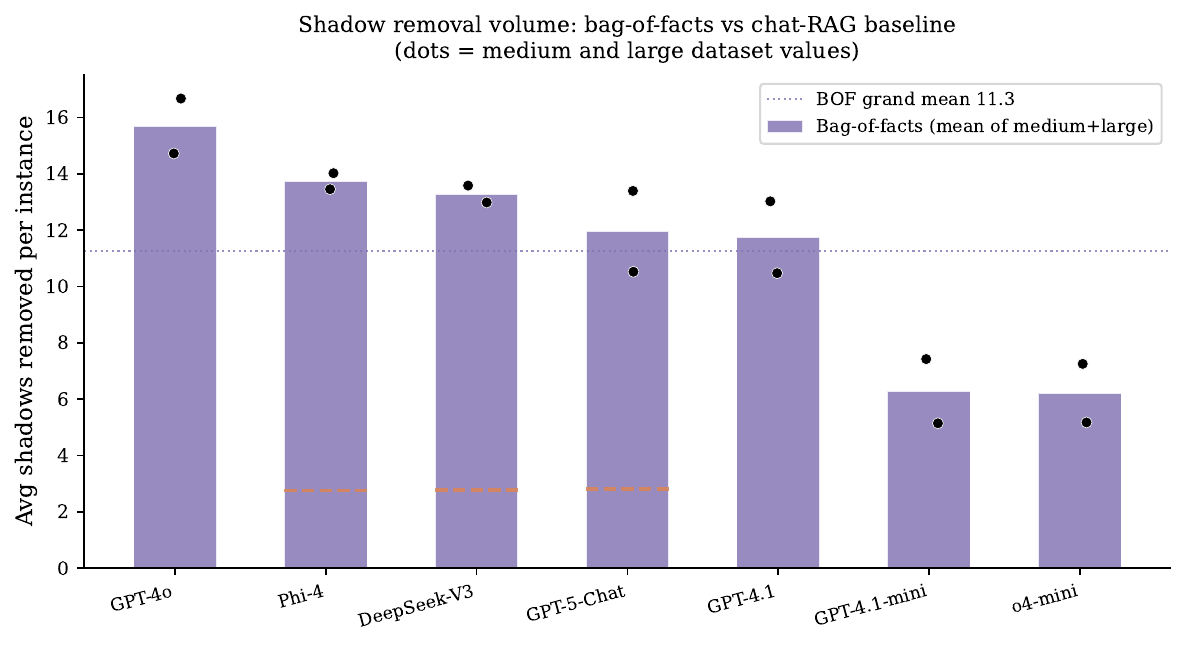}
\caption{Distribution of shadow counts identified by $\kappa$ in the bag-of-facts regime, by model.}
\label{fig:bof_shadows_distribution}
\end{figure}

\begin{figure}[H]
\centering
\includegraphics[width=\linewidth]{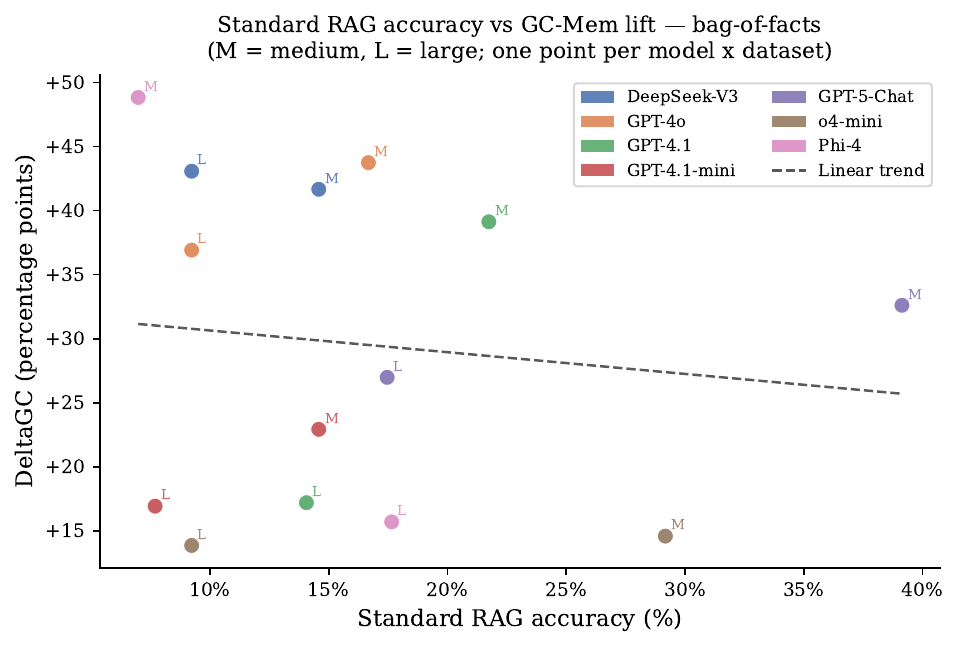}
\caption{Standard RAG accuracy vs.\ $\Delta$GC in bag-of-facts: weaker baselines often coincide with more headroom for conflict pruning.}
\label{fig:bof_std_vs_dgc}
\end{figure}

\clearpage

\section{Mathematical Proofs}

\subsection{Proof of Theorem 1: Asymptotic Recall Decay}
\begin{proof}
Let the total memory stream contain $N$ chunks asserting the stale state $S_{old}$ and $M$ chunks asserting the valid state $S_{new}$. Under the Accumulation Hypothesis (Assumption 3.3), $N \gg M$. 

We model the retrieval of top-$k$ chunks $\mathcal{R}_k$ as sampling without replacement from the population $N+M$. Let $X$ be the random variable denoting the number of valid chunks ($S_{new}$) successfully retrieved. If retrieval were purely uniform, $X$ would follow a hypergeometric distribution:
\begin{equation}
P(X = x) = \frac{\binom{M}{x}\binom{N}{k-x}}{\binom{N+M}{k}}
\end{equation}
The expected value is $\mathbb{E}[X] = k \frac{M}{N+M}$.

However, dense retrievers utilize cosine similarity over high-dimensional embeddings. Under Semantic Equivalence (Lemma~\ref{lem:equivalence}), $|sim(q, S_{old}) - sim(q, S_{new})| < \epsilon$. Because $S_{old}$ constitutes the vast majority of the embedded semantic space, dense retrieval exhibits a structural bias towards the denser centroid of $S_{old}$ phrasings. Therefore, the hypergeometric expected value serves as a strict theoretical upper bound.

Taking the limit as the interaction history scales:
\begin{equation}
\lim_{N \to \infty} \mathbb{E}[X] \approx \lim_{N \to \infty} k \frac{M}{N+M} = 0
\end{equation}
As $N \to \infty$, the probability of retrieving zero valid evidence ($P(X=0)$) approaches 1, proving Asymptotic Recall Decay.
\end{proof}

\subsection{Proof of Theorem 2: The Majority Vote Trap}
\begin{proof}
Assume the retriever successfully circumvents Theorem 1 and retrieves a mixed context $\mathcal{R}_k$ containing $1$ valid chunk ($S_{new}$) and $k-1$ stale chunks ($S_{old}$). 

During generation, the LLM utilizes scaled dot-product attention. Let $\alpha_i$ denote the attention weight assigned to chunk $c_i \in \mathcal{R}_k$. Under Lemma~\ref{lem:equivalence}, the semantic embeddings of $S_{old}$ and $S_{new}$ are equivalent up to $\epsilon < 0.10$. Consequently, the query vector $q$ yields roughly identical dot-products across all $k$ chunks. 

Applying the softmax function over these practically uniform logits results in a non-discriminating attention distribution:
\begin{equation}
\alpha_i \approx \frac{1}{k} \quad \forall c_i \in \mathcal{R}_k
\end{equation}

The final generation probability for a state $y$ is the sum of the attention mass corresponding to the chunks asserting that state. 
\begin{equation}
P(y = S_{old}) = \sum_{c_i \in S_{old}} \alpha_i \approx \frac{k-1}{k}
\end{equation}
\begin{equation}
P(y = S_{new}) = \sum_{c_j \in S_{new}} \alpha_j \approx \frac{1}{k}
\end{equation}

For any typical retrieval window ($k \ge 5$), $\frac{k-1}{k} \gg \frac{1}{k}$. The probability mass overwhelmingly favors the stale state. Thus, even when valid evidence is perfectly retrieved, generation collapses due to attention dilution, proving the Majority Vote Trap.
\end{proof}

\end{document}